%% file: main.tex
\documentclass[lettersize,journal,twoside]{IEEEtran}
\usepackage{amsmath,amsfonts}
\usepackage{array}
\usepackage[caption=false,font=normalsize,labelfont=sf,textfont=sf]{subfig}
\usepackage{textcomp}
\usepackage{stfloats}
\usepackage{url}
\usepackage{verbatim}
\usepackage{graphicx}
\usepackage{cite}
\usepackage{hyperref}
 
\usepackage{booktabs} 
\usepackage{amsfonts}
\usepackage{nicefrac}
\usepackage{microtype} 
\usepackage[dvipsnames]{xcolor}
\usepackage{enumitem}
\usepackage{amsthm}
\usepackage{amsmath}
\usepackage{bm}

\usepackage{amssymb}
\usepackage{mathtools}
\usepackage[linesnumbered,ruled,vlined]{algorithm2e}
\usepackage{makecell}
\usepackage{multirow}
\usepackage{booktabs}
\usepackage{wrapfig}
\usepackage{extarrows}
\usepackage{colortbl}
\usepackage{tikz}
\usepackage{pifont}
\usepackage[most]{tcolorbox}
\usepackage{balance}
\usepackage{adjustbox}
\usepackage{wasysym}
\usepackage{textcase}

\newcommand{\modelname}{DyGFM}

\newcommand{\eg}{\textit{e.g.}}

\newcommand{\etc}{\textit{etc.}}

\theoremstyle{definition}

\newtheorem{prop}{Proposition}

\newtcolorbox{mathbox}{
  colframe=black,
  colback=white,
  sharp corners,
  boxrule=0.5pt,
  breakable,
  left=1.5pt, right=1.5pt, top=1.5pt, bottom=1.5pt
}

\begin{document}

\title{Decoupled and Divergence-Conditioned Prompt for Multi-domain Dynamic Graph Foundation Models}

\author{Haonan~Yuan,~Qingyun~Sun,~\IEEEmembership{Member,~IEEE},~Junhua~Shi,~Xingcheng~Fu,~\IEEEmembership{Member,~IEEE},\\Jianxin~Li,~\IEEEmembership{Senior Member,~IEEE},~and~Philip~S.~Yu,~\IEEEmembership{Life~Fellow,~IEEE}
\thanks{H. Yuan, Q. Sun, J. Shi, and J. Li are with the School of Computer Science and Engineering, Beihang University, Beijing 100191, China (Email: \{yuanhn, sunqy, shijunhua, lijx\}@buaa.edu.cn).}
\thanks{X. Fu is with the Key Lab of Education Blockchain and Intelligent Technology, Ministry of Education, Guangxi Normal University, Guilin 541004, China (E-mail: fuxc@gxnu.edu.cn).}
\thanks{P. S. Yu is with the Department of Computer Science, University of Illinois at Chicago, Chicago 60607, USA (E-mail: psyu@uic.edu).}
\thanks{Manuscript received May 11, 2026; revised June 18, 2026.}}

\markboth{IEEE TRANSACTIONS ON PATTERN ANALYSIS AND MACHINE INTELLIGENCE,~Vol.~16, No.~8, August~2026}{Yuan \MakeLowercase{\textit{et al.}}: Decoupled and Divergence-Conditioned Prompt for Multi-domain Dynamic Graph Foundation Models}


\maketitle

\begin{abstract}
  Dynamic graphs are ubiquitous in real-world systems, and building generalizable dynamic Graph Foundation Models has become a frontier in graph learning. However, dynamic graphs from different domains pose fundamental challenges to unified modeling, as their semantic and temporal patterns are inherently inconsistent, making the multi-domain pre-training difficult. Consequently, the widely used ``pretrain-then-finetune'' paradigm often suffers from severe negative knowledge transfer. To the best of our knowledge, there exists no multi-domain dynamic GFM.
  In this work, we propose \textbf{\modelname}, a \underline{\textbf{\smash{Dy}}}namic \underline{\textbf{G}}raph \underline{\textbf{F}}oundation \underline{\textbf{M}}odel over multiple domains based on decoupled and divergence-conditioned prompting.
  To disentangle transferable semantics from the domain-specific dynamics, we introduce a dual-branch pre-training strategy with semantic-temporal decoupling.
  To alleviate negative transfer during domain adaptation, we further develop a cross-domain routing mechanism with divergence-aware expert selection.
  To enable efficient downstream fine-tuning, we design a divergence-conditioned prompt generator that injects lightweight, learnable graph prompts tailored to semantic and temporal traits.
  Extensive experiments on continuous dynamic graph benchmarks demonstrate that \modelname~consistently outperforms 12 state-of-the-art baselines on both node classification and link prediction tasks, achieving superior effectiveness and efficiency.
\end{abstract}

\begin{IEEEkeywords}
  Graph foundation models, dynamic graph representation learning, multi-domain graph pre-training, and graph prompt learning.
\end{IEEEkeywords}

\input{1_intro}
\input{2_related_work}
\input{3_notation}
\input{4_method}
\input{5_experiment}
\input{6_conclusion}

\appendices

\input{appendix/1_notations}




\bibliographystyle{IEEEtran}
\bibliography{ref.bib}

\begin{IEEEbiography}[{\includegraphics[width=1in,height=1.2in,clip,keepaspectratio]{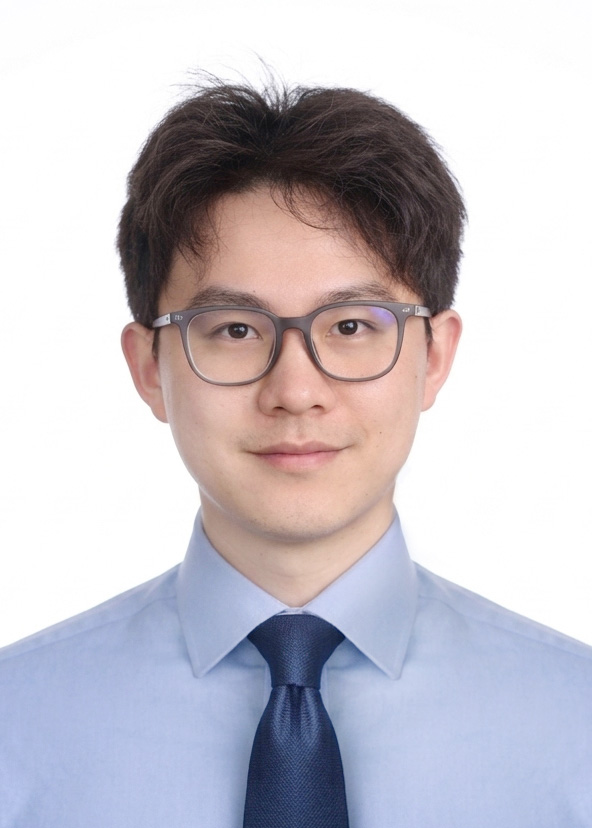}}]{Haonan Yuan} is currently a Ph.D. candidate at the Beijing Advanced Innovation Center for Big Data and Brain Computing at Beihang University. His research interests include graph foundation models, dynamic graph learning, and OOD generalization. He has published several papers on IEEE TPAMI, ICML, NeurIPS, ICLR, WWW, AAAI, \textit{etc}.
\end{IEEEbiography}

\begin{IEEEbiography}
[{\includegraphics[width=1in,height=1.2in,clip,keepaspectratio]{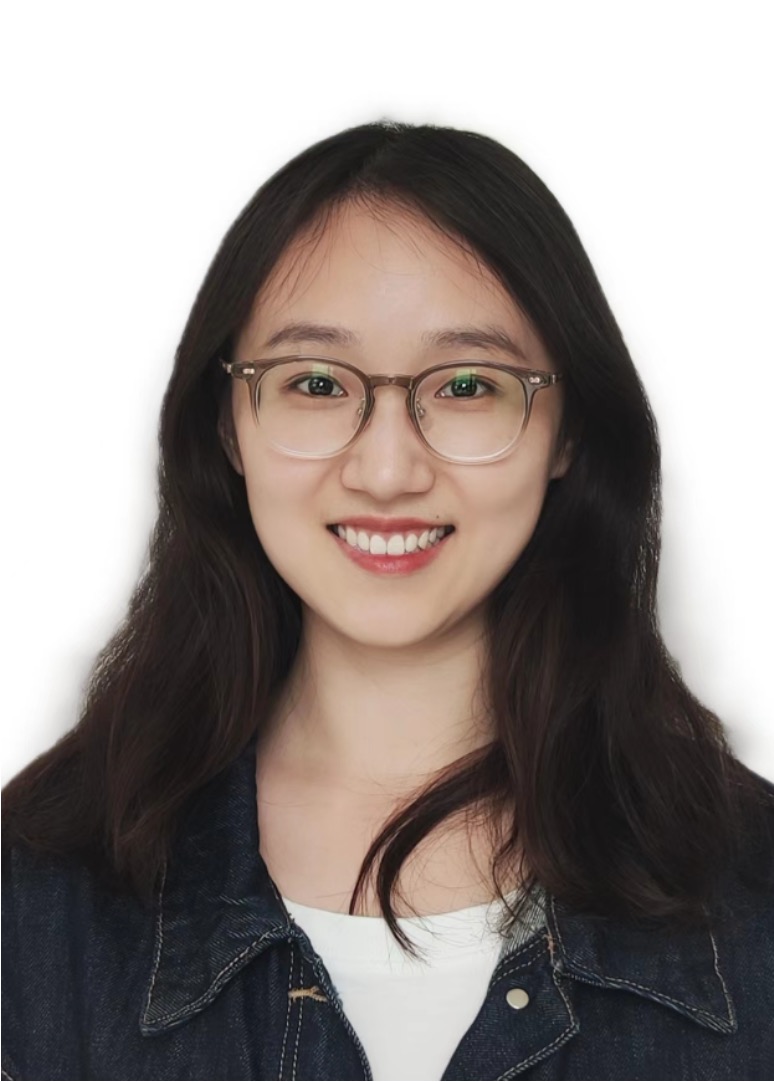}}]{Qingyun Sun} is currently an Associate Professor at the School of Computer Science and Engineering at Beihang University. Her research interests include graph machine learning and data mining. She has published several papers on IEEE TPAMI, IEEE TKDE, WWW, AAAI, ICDM, CIKM, \textit{etc}.
\end{IEEEbiography}

\begin{IEEEbiography}[{\includegraphics[width=1in,height=1.2in,clip,keepaspectratio]{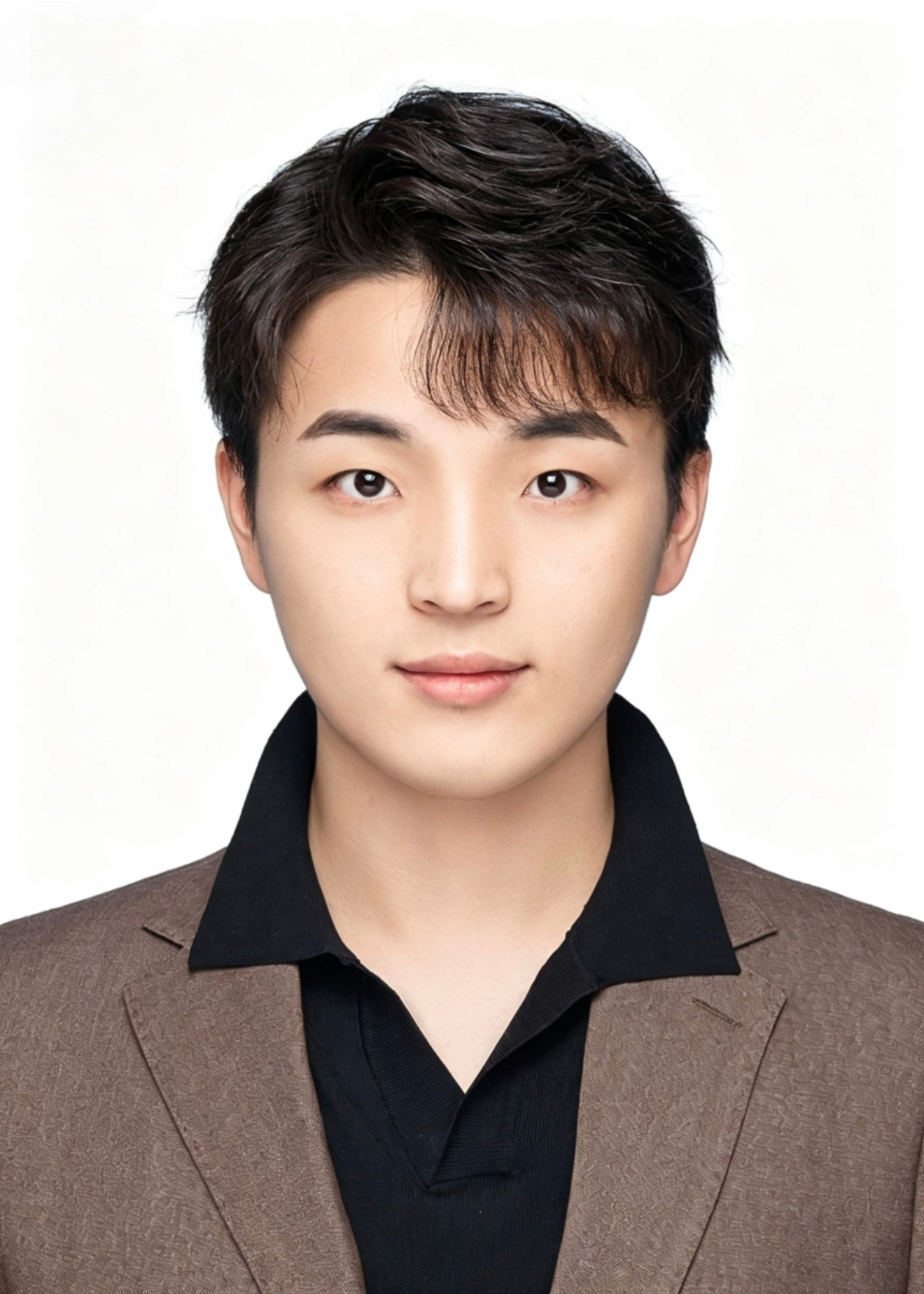}}]{Junhua Shi} is currently a Ph.D. candidate at the Beijing Advanced Innovation Center for Big Data and Brain Computing at Beihang University. His research interests include trustworthy graph foundation models.
\end{IEEEbiography}

\begin{IEEEbiography}[{\includegraphics[width=1in,height=1.2in,clip,keepaspectratio]{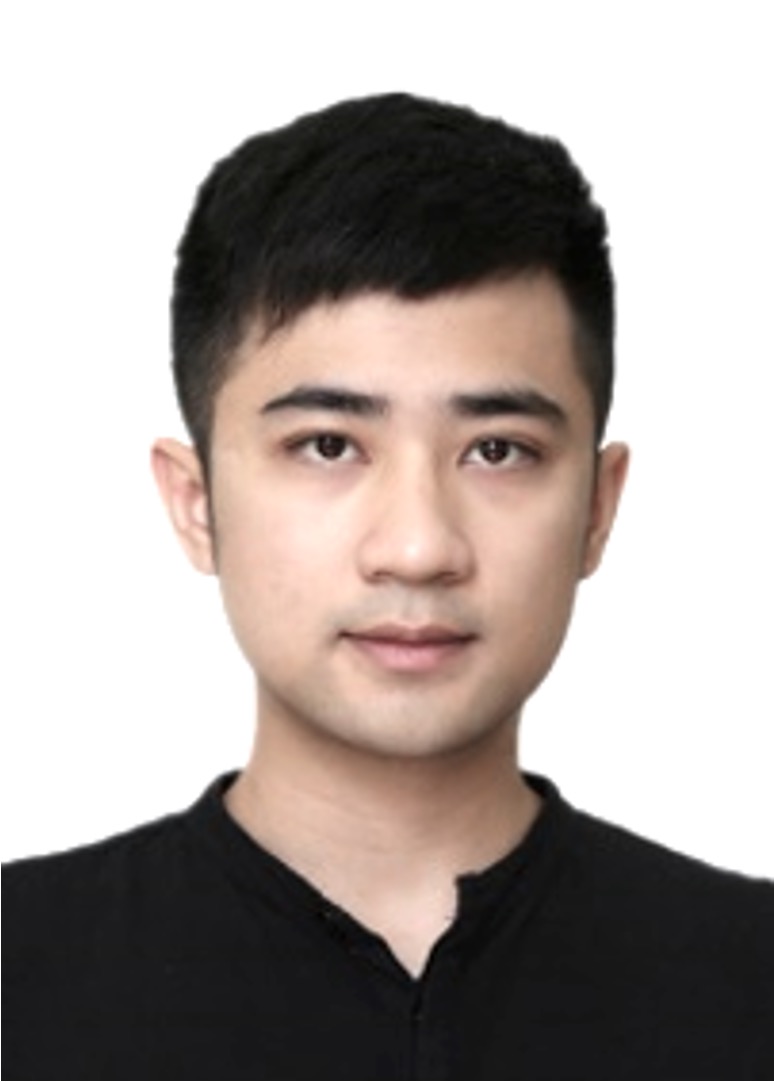}}]{Xingcheng Fu}  is currently an Associate Professor at the Key Lab of Education Blockchain and Intelligent Technology at Guangxi Normal University. His research interests include graph representation learning, complex networks, and social network analysis. He has published several papers on IEEE TKDE, WWW, AAAI, ICDM, CIKM, \textit{etc}.
\end{IEEEbiography}

\begin{IEEEbiography}
[{\includegraphics[width=1in,height=1.2in,clip,keepaspectratio]{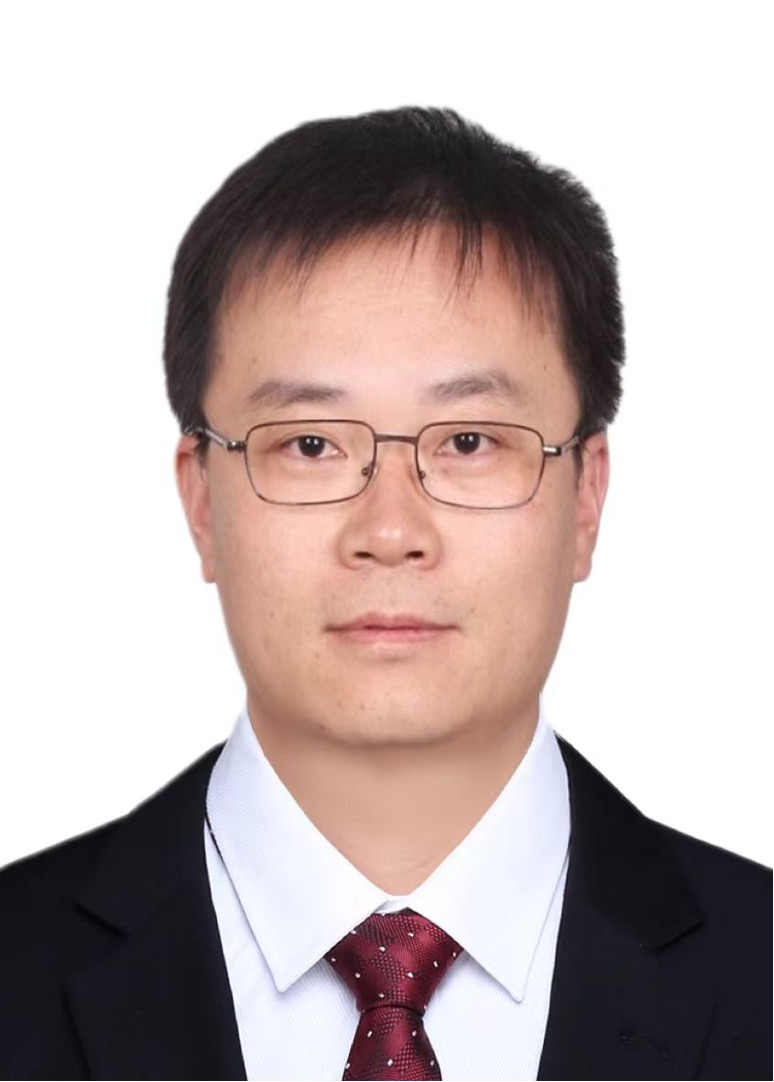}}]{Jianxin Li} is currently a Professor with the School of Computer Science and Engineering, and Beijing Advanced Innovation Center for Big Data and Brain Computing in Beihang University. His current research interests include social networks, machine learning, big data, and trustworthy computing. Dr. Li has published research papers in top-tier journals and conferences, including the IEEE TKDE, TDSC, JAIR, ACM TOIS, TKDD, KDD, AAAI, WWW, \textit{etc}. 
\end{IEEEbiography}

\begin{IEEEbiography}
[{\includegraphics[width=1in,height=1.2in, clip,keepaspectratio]{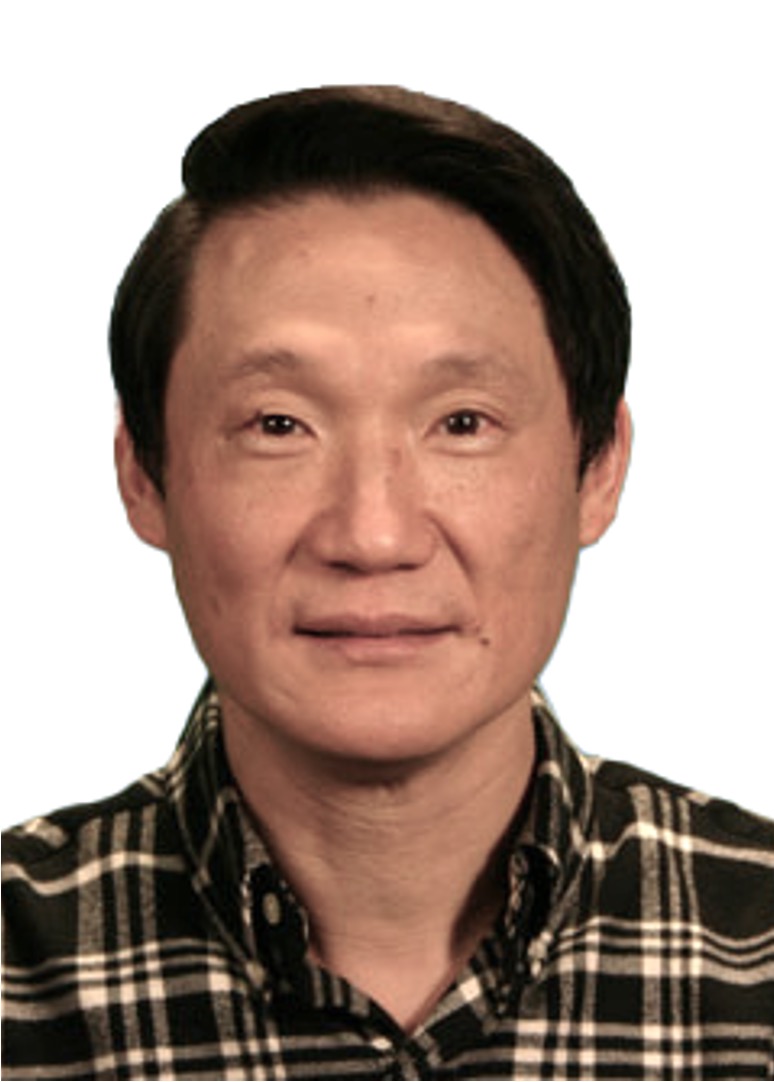}}]{Philip S. Yu} is a Distinguished Professor and the Wexler Chair in Information Technology at the Department of Computer Science, University of Illinois at Chicago. Before joining UIC, he was at the IBM Watson Research Center, where he built a world-renowned data mining and database department. He is a Fellow of the ACM and IEEE. Dr. Yu was the Editor-in-Chief of ACM TKDD (2011-2017) and IEEE TKDE (2001-2004).
\end{IEEEbiography}

\vfill

\end{document}

%% file: 1_intro.tex
\section{Introduction}
\label{sec:intro}
Graphs provide a powerful abstraction for modeling complex relationships among entities, and have been widely adopted in diverse domains, including social networks~\cite{newman2002random, fan2019graph, sankar2021graph}, recommendation systems~\cite{ying2018graph, wu2019session, wang2019knowledge}, knowledge graphs~\cite{wang2014knowledge, wang2017knowledge, zhang2022knowledge}, and biological systems~\cite{lin2021kgnn, gligorijevic2021structure, jha2022prediction}, \etc~In many real-world scenarios, node interactions are inherently temporal, with graph structures and node features evolving continuously over time. Such dynamic graphs are commonly characterized by asynchronous edge events~\cite{fan2022towards, wen2022trend, liu2023tmac}, time-sensitive node states~\cite{xu2019spatio, singer2019node, huang2023temporal}, and evolving topological patterns~\cite{zhang2020spatio, jin2022neural, yang2023time}. To capture these temporal dependencies, Dynamic Graph Neural Networks (DGNNs) have emerged as a major paradigm, enabling temporal encoding, event-aware message propagation, and inductive generalization over time-evolving data streams~\cite{xhonneux2020continuous, skarding2021foundations, zhang2022dynamic}. Recent advances in DGNNs have further enabled effective learning on continuous-time dynamic graphs, with applications in user behavior prediction~\cite{wang2020calendar, liu2024tp, mao2025boosting}, fraud detection~\cite{lu2022bright, duan2024dga, dong2024dynamic}, and molecular interaction modeling~\cite{zheng2020predicting, li2022graph, li2025spatial}, \etc

Despite their success, existing DGNNs are typically developed under task- or domain-specific training protocols, and thus often struggle to generalize across domains and tasks~\cite{chen2024prompt, yu2025nodetime}. This limitation resembles the early development of language~\cite{chang2024survey} and vision~\cite{zhang2024vision} models before the emergence of foundation models~\cite{achiam2023gpt, khan2022transformers}. In the graph domain, Graph Foundation Models (GFMs) have recently shown that large-scale pre-training on diverse static graphs can learn transferable representations that generalize well across various downstream tasks~\cite{mao2024position, shi2024graph, zi2024prog}. Instead of training a separate model for each task, an ideal GFM is pre-trained once on multi-domain graphs and then adapted to different tasks with limited supervision~\cite{shi2024lecture, bommasani2021opportunities}, following the widely adopted ``pretrain-then-finetune'' paradigm~\cite{tang2024graphgpt, he2024unigraph, lachi2024graphfm}. Inspired by these advances, researchers have begun to explore dynamic graph foundation models, which aim to extend this paradigm from static graphs to time-evolving graphs and support multi-domain, multi-task adaptation~\cite{yu2025nodetime}.
However, building a dynamic graph foundation model over multiple domains remains highly challenging. Unlike static graphs, which can often be integrated through structural unification or feature-space alignment, dynamic graphs carry domain-specific semantics, incompatible temporal granularities, and divergent evolution patterns. These discrepancies make it difficult to learn unified representations from multiple dynamic graphs during pre-training. Moreover, when the pre-trained model is adapted to downstream domains, mismatched semantic or temporal patterns may introduce irrelevant source-domain knowledge and lead to negative transfer. \textbf{Specifically:}


\begin{figure}
    \centering
    \includegraphics[width=1\linewidth]{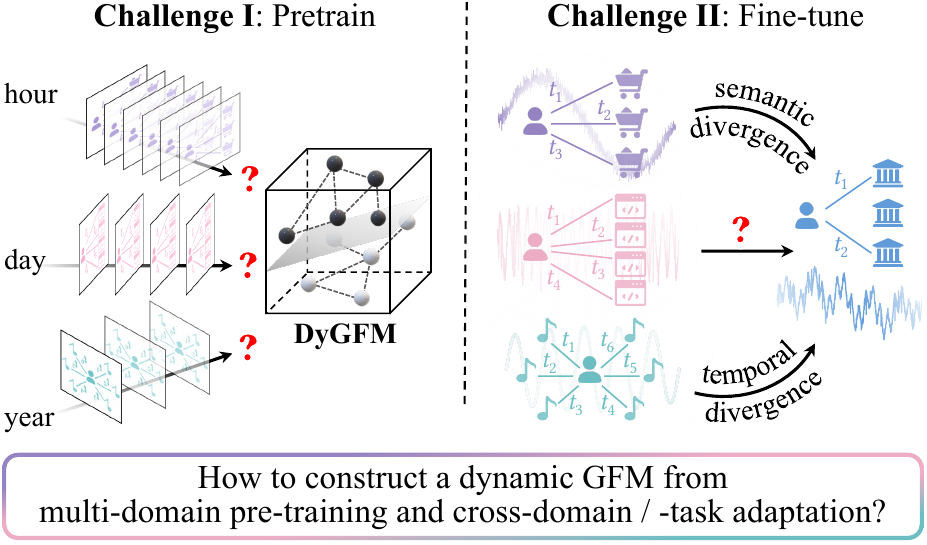}
    \caption{Challenges of constructing a dynamic GFM.}
    \label{fig:intro_compare}
    \vspace{-0.35cm}
\end{figure}

\textbf{Challenge I: Misaligned temporal semantics across domains make multi-domain pre-training difficult.}
Unlike static graphs that can be integrated through structural unification, dynamic graphs encode temporal order, relative intervals, and event causality as essential parts of their semantics. For example, user-item interactions in an e-commerce platform occur at millisecond-level timestamps, whereas citation networks typically evolve over years. Since most dynamic graph models rely on relative time encodings or continuous-time propagation mechanisms, directly combining domains with different temporal granularities may introduce ambiguity and distort the original semantics. Without a shared or aligned notion of time, heterogeneous dynamic graphs cannot be simply fused into a unified pre-training corpus. Consequently, the idea of scaling via data aggregation, which is central to the success of static GFMs, becomes non-trivial in dynamic settings.

\textbf{Challenge II: Semantic-temporal divergence across domains leads to negative transfer during adaptation.}
Compared with the static settings, dynamic graphs differ not only in interaction patterns and feature distributions, but also in their temporal evolution patterns. For instance, a user churn dataset may reflect short-term behavioral bursts, while a co-authorship network usually captures long-term collaboration trends. When the widely adopted ``pretrain-then-finetune'' paradigm is applied to such heterogeneous dynamic domains, the model may transfer temporal priors or semantic cues from irrelevant source domains, thereby degrading downstream performance on the target domain. This issue becomes particularly pronounced under limited supervision, where the model has insufficient target-domain labels to correct misleading inductive biases. Without explicitly modeling domain divergence, a unified dynamic graph foundation model may generalize poorly and even negatively affect domain-specific adaptation.

Several recent works have attempted to improve generalization in dynamic graph learning through prompt-based adaptation towards establishing DyGFMs. For instance, DyGPrompt~\cite{yu2025nodetime} introduces a node- and timestamp-conditioned prompt generator, allowing fine-tuning of pre-trained temporal graph models on downstream tasks. Similarly, TIGPrompt~\cite{chen2024prompt} designs discrete temporal prompts for multi-task generalization. While effective within single-domain settings, these methods generally assume that time semantics remain consistent across all training and testing data. As a result, they do not account for the structural and temporal divergence that naturally arises across multiple heterogeneous domains. Moreover, most existing approaches rely on end-to-end fine-tuning of model backbones or restrict prompting to shallow embeddings, thereby limiting scalability and flexibility, which remain major bottlenecks in building practical dynamic GFMs.

To tackle the aforementioned challenges, we propose \textbf{\modelname}, a \underline{\textbf{Dy}}namic \underline{\textbf{G}}raph \underline{\textbf{F}}oundation \underline{\textbf{M}}odel over multiple domains by decoupled and divergence-conditioned prompting. \modelname~is built upon three key modules.
To disentangle transferable semantics from domain-specific dynamics, we propose a dual-branch pre-training strategy with semantic-temporal decoupling, which separately captures static feature semantics and fine-grained temporal patterns.
To alleviate negative transfer during domain adaptation, we design a cross-domain routing mechanism with divergence-aware expert selection, which selectively aggregates relevant source-domain experts according to semantic and temporal discrepancies.
To enable efficient downstream fine-tuning, we further introduce a divergence-conditioned prompt generator that injects lightweight, learnable graph prompts tailored to domain-specific semantic and temporal patterns.
\textbf{Our contributions are:}
\begin{itemize}[leftmargin=1.5em]
    \item We propose a multi-domain dynamic graph foundation model named \modelname. To the best of our knowledge, this is the first work that successfully builds a multi-domain dynamic GFM by explicitly addressing both multi-domain pre-training and negative knowledge transfer.
    \item Its architecture consists of semantic-temporal decoupled pre-training, divergence-aware expert routing, and divergence-conditioned prompting. This design enables \modelname~to handle temporal semantic incompatibility during pre-training and mitigate negative transfer during domain adaptation.
    \item Extensive experiments on continuous dynamic graph benchmarks demonstrate that \modelname~consistently outperforms 12 state-of-the-art baselines in node classification and link prediction, achieving superior effectiveness and efficiency.
\end{itemize}

%% file: 2_related_work.tex
\section{Related Work}
\label{sec:related_work}

\subsection{Dynamic Graph Learning}
Dynamic graph learning aims to learn representations that evolve over time. Unlike static graphs, dynamic graphs require models to capture both structural and temporal dependencies induced by timestamped interactions, event intervals, and node states. Existing works can be grouped into structure-oriented, temporal-oriented, and application-oriented methods.

\textbf{Structure-oriented methods.}
Structure-oriented methods mainly focus on preserving evolving graph topology and capturing time-dependent structural patterns. One representative line is the temporal random walk-based methods. CTDNE~\cite{nguyen2018continuous} extends random walks to continuous-time dynamic graphs by enforcing chronological constraints during walk sampling. CAWs~\cite{wanginductive} further introduces causal anonymous walks to improve inductive temporal representation learning. NeurTW~\cite{jin2022neural} enhances temporal walks with neural message passing and motif-aware temporal structures. PINT~\cite{souza2022provably} studies the expressiveness of temporal graph networks from the perspective of temporal walks, while TPNet~\cite{lu2024improving} improves temporal link prediction through temporal walk matrix projection. Another representative line is the temporal neighborhood-based methods. TGAT~\cite{xu2020inductive} introduces time encoding and attention-based temporal neighborhood aggregation. GraphMixer~\cite{congwe2023} uses simple MLP-based architectures to aggregate temporal neighborhood information efficiently. DyGFormer~\cite{yu2023towards} tokenizes historical interactions and models them with transformer-style sequence encoding. CNE-N~\cite{cheng2024co} designs co-neighbor encoding for efficient dynamic link prediction, while SEAN~\cite{zhang2024towards} and RepeatMixer~\cite{zou2024repeat} further improve temporal neighborhood modeling through adaptive neighborhood selection and repeat-aware sampling.
These methods are effective in modeling local temporal structures and dynamic neighborhood dependencies, but they usually rely on sampled walks, historical neighborhoods, or local interaction windows, which may limit their scalability and long-range temporal modeling ability.

\textbf{Temporal-oriented methods.}
Temporal-oriented methods emphasize the modeling of temporal dependency itself, especially under continuous-time and irregular interaction settings. RNN-based methods such as JODIE~\cite{li2019predicting} and DyGNN~\cite{ma2020streaming} update node representations sequentially according to streaming interactions, thereby capturing the evolving states of nodes over time. RTRGN~\cite{chen2023recurrent} further revises temporal information during recurrent aggregation to better incorporate historical neighbor states. Temporal point process-based methods, including HTNE~\cite{zuo2018embedding}, DyRep~\cite{trivedi2019dyrep}, M2DNE~\cite{lu2019temporal}, TREND~\cite{wen2022trend}, and EasyDGL~\cite{chen2024easydgl}, model the occurrence of future interactions through conditional intensity functions, which enable them to capture asynchronous events and temporal decay effects. Memory-based methods such as TGN~\cite{rossi2020temporal}, NAT~\cite{luo2022neighborhood}, TIGER~\cite{zhang2023tiger}, PRES~\cite{supres2020}, MemMap~\cite{ji2024memmap}, and MSPipe~\cite{sheng2024mspipe} maintain node-level memory states and update them through event-driven message passing, supporting online representation learning over streaming graphs. More recently, frequency-domain methods such as FreeDyG~\cite{tian2024freedyg} and BandRank~\cite{li2025ranking} transform temporal signals into the frequency space to capture periodic patterns, long-term trends, and multi-scale temporal variations. These methods provide powerful tools for modeling temporal evolution, but the learned temporal patterns are often tied to specific datasets or domains. As a result, directly transferring with different temporal resolutions, event generation mechanisms, or interaction patterns remains non-trivial.

\textbf{Application-oriented methods.}
Application-oriented methods adapt dynamic graph learning to specific downstream scenarios. For sequential recommendation, methods such as TGSRec~\cite{fan2021continuous}, DGSR~\cite{zhang2022dynamic}, PTGCN~\cite{huang2023position}, TCGC~\cite{tang2025tcgc}, NeuFilter~\cite{xia2024neural}, and TGCL4SR~\cite{zhang2024temporal} model dynamic user-item interactions and evolving preferences. For anomaly detection and community discovery, methods such as TagGen~\cite{zhou2020data}, SAD~\cite{tian2023sad}, GeneralDyG~\cite{yang2025generalizable}, CDGP~\cite{ji2023community}, TGC~\cite{liudeep2024}, and DyG-MF~\cite{li2025revisiting} capture abnormal temporal behaviors or evolving community structures. Although these methods achieve strong performance in their target applications, their architectures and objectives are usually task-specific, offering limited support for unified pre-training and cross-domain adaptation.

Overall, existing dynamic graph learning methods have advanced the modeling of temporal interactions and evolving structures. However, most of them are developed for single-domain or task-specific scenarios, with limited consideration of multi-domain pre-training, temporal scale inconsistency, and negative transfer during cross-domain adaptation. These limitations motivate our study on dynamic graph foundation modeling over multiple domains.

\subsection{Static Graph Foundation Models.}
Graph Foundation Models (GFMs) aim to learn transferable knowledge from large-scale graph data and generalize to diverse downstream tasks and domains. According to the backbone architecture, existing static GFMs can be broadly grouped into GNN-based GFMs, LLM-based GFMs, and hybrid GFMs.

\textbf{GNN-based GFMs.}
GNN-based GFMs use graph-native encoders as the main backbone, where transferable knowledge is learned through message passing, graph transformers, structural tokenization, or domain-aware representation alignment. Early prompt-based methods, such as GraphPrompt~\cite{liu2023graphprompt}, HGPrompt~\cite{hgprompt}, MultiGPrompt~\cite{multigprompt}, and UniPrompt~\cite{uniprompt}, adapt pre-trained GNN encoders to downstream tasks through learnable prompts or unified task templates. Subsequent multi-domain methods, such as SAMGPT~\cite{yu2025samgpt}, GCOPE~\cite{zhao2024all}, MDGPT~\cite{mdgpt}, MDGCL~\cite{mdgcl}, BRIDGE~\cite{yuanmuch}, and MDGFM~\cite{mdgfm}, further improve cross-domain transfer by introducing contrastive alignment, domain tokens, source-domain selection, or spectral regularization. More recent methods explore specialized graph-native architectures. For example, GMoPE~\cite{gmope} introduces expert prompts for adaptive transfer, while GRAVER~\cite{graver} constructs transferable graph vocabularies through hierarchical routing. In addition, GraphMoRE~\cite{graphmore}, RiemannGFM~\cite{riemanngfm}, GraphGlue~\cite{graphglue}, and CRGFM~\cite{CRGFM} incorporate geometric or Riemannian modeling to better capture complex structural patterns. These methods provide effective graph-native transfer mechanisms, but they are mainly designed for static graphs and do not explicitly model temporal order, time granularity, or dynamic evolution across domains.

\textbf{LLM-based GFMs.}
LLM-based GFMs reformulate graph learning into language-compatible formats and use large language models as the primary backbone. Representative methods include LangGFM~\cite{langgfm} and PromptGFM~\cite{promptgfm}, which textualize graph structures, node attributes, or graph tasks into language-readable inputs and perform graph reasoning or prediction in the language space. Other methods further convert graphs into structured formats such as JSON, XML, or tables, enabling graph data to reuse the general reasoning and instruction-following abilities of pretrained foundation models~\cite{eremeev2025turning}. This line benefits from the semantic knowledge and flexible task interfaces of LLMs, especially for text-attributed graphs or graph-language tasks. However, LLM-based GFMs often rely on textualized graph descriptions and may lose fine-grained structural or temporal information during the transformation process. Therefore, they are not naturally suited for continuous-time dynamic graphs, where irregular event intervals and temporal dependencies are central to representation learning.

\textbf{GNN-LLM Hybrid GFMs.}
Hybrid GFMs combine graph encoders with language encoders or LLMs, aiming to exploit both graph-native structural modeling and language-level semantic reasoning. OFA~\cite{ofa} and UniGraph~\cite{he2024unigraph} provide early attempts to unify graph tasks through graph-language interfaces. GraphGPT~\cite{tang2024graphgpt} and LLaGA~\cite{llaga} connect graph encoders with LLMs through projection modules, while GraphCLIP~\cite{graphclip} and BooG~\cite{boog} align graph representations with textual supervision or class-level semantic information. UniGraph2~\cite{unigraph2} further extends this paradigm to multimodal graph settings, and GOFA~\cite{gofa} integrates graph computation with frozen language models for graph-aware generation. Compared with purely GNN-based or LLM-based methods, hybrid GFMs offer a more flexible interface for combining structural and semantic knowledge. Nevertheless, most hybrid GFMs still focus on static graphs or text-attributed graphs, and their pre-training and adaptation mechanisms do not explicitly address temporal scale inconsistency, domain-specific evolution patterns, or negative transfer in multi-domain dynamic graph learning.

Overall, static GFMs have significantly advanced transferable graph representation learning through graph-native pre-training, language-based reformulation, and graph-language hybrid architectures. However, they generally lack explicit mechanisms for modeling continuous temporal evolution and aligning heterogeneous time semantics across domains. As a result, directly extending static GFMs to dynamic graph scenarios may lead to temporal semantic distortion during pre-training and negative transfer during cross-domain adaptation. This motivates the development of dynamic graph foundation models that can jointly support multi-domain pre-training, semantic-temporal decoupling, and efficient downstream fine-tuning.

\input{table/compare}

\subsection{Dynamic Graph Foundation Models.}
Compared with static GFMs, dynamic graph foundation models remain much less explored. Existing related studies mainly focus on prompt-based adaptation for pre-trained dynamic graph models rather than building foundation models over multiple dynamic graph domains. TIGPrompt~\cite{chen2024prompt} introduces a temporal prompt generator to bridge the temporal and semantic gaps between pre-training and downstream prediction on temporal interaction graphs, and supports lightweight prompt tuning over frozen TIG backbones. DyGPrompt~\cite{yu2025nodetime} further develops a pre-training and prompt learning framework for dynamic graphs, where node and time prompts are used to alleviate the discrepancy between link-prediction pre-training and downstream tasks, and dual condition-nets are designed to capture evolving node-time patterns. More recently, DDGPrompt~\cite{peng2025data} proposes a data-centric prompt tuning strategy that refines pre-trained node embeddings through temporal bias, edge weight, and feature mask prompts, aiming to improve few-shot adaptability across downstream tasks. Although these studies improve the adaptability of dynamic graph models, they are still mainly centered on single-domain prompt tuning or task-level transfer. They do not explicitly address multi-domain pre-training over dynamic graphs with incompatible temporal semantics, nor do they model semantic-temporal divergence between source and target domains for negative transfer mitigation. Therefore, to the best of our knowledge, multi-domain dynamic graph foundation modeling remains largely unexplored.

%% file: table/compare.tex
\begin{table*}[!t]
\setlength{\tabcolsep}{2.5pt}
  \centering
  \caption{Comparison with representative Graph Foundation Models.}
  \resizebox{\textwidth}{!}{
    \begin{tabular}{llcccccc}
    \toprule
    \textbf{Category} & \textbf{Representative Methods} & \makecell{\textbf{Dynamic}\\[2pt]\textbf{Graph Modeling}} & \makecell{\textbf{Multi-domain}\\[2pt]\textbf{Pre-training}} & \makecell{\textbf{Temporal}\\[2pt]\textbf{Alignment}} & \makecell{\textbf{Cross-domain}\\[2pt]\textbf{Adaptation}} & \makecell{\textbf{Negative Transfer}\\[2pt]\textbf{Mitigation}} & \makecell{\textbf{Efficient}\\[2pt]\textbf{Fine-tuning}} \\
    \midrule
    GNN-based Static GFMs & \makecell[l]{GraphPrompt~\cite{liu2023graphprompt}, SAMGPT~\cite{yu2025samgpt},\\[2pt]GCOPE~\cite{zhao2024all}, BRIDGE~\cite{yuanmuch},\\[2pt]MDGFM~\cite{mdgfm}, GRAVER~\cite{graver}} & $\ocircle$     & $\CIRCLE$     & $\ocircle$     & $\CIRCLE$     & $\RIGHTcircle$     & $\CIRCLE$ \\
    \midrule
    LLM-based Static GFMs & LangGFM~\cite{langgfm}, PromptGFM~\cite{promptgfm} & $\ocircle$     & $\ocircle$     & $\ocircle$     & $\RIGHTcircle$     & $\ocircle$     & $\RIGHTcircle$ \\
    \midrule
    Hybrid Static GFMs & \makecell[l]{OFA~\cite{ofa}, GraphGPT~\cite{tang2024graphgpt},\\[2pt]LLaGA~\cite{llaga}, GraphCLIP~\cite{graphclip},\\[2pt]GOFA~\cite{gofa}} & $\ocircle$     & $\CIRCLE$     & $\ocircle$     & $\RIGHTcircle$     & $\ocircle$     & $\RIGHTcircle$ \\
    \midrule
    \multirow{3}[4]{*}{Dynamic GFMs} & TIGPrompt~\cite{chen2024prompt}, DDGPrompt~\cite{peng2025data} & $\CIRCLE$     & $\ocircle$     & $\ocircle$     & $\ocircle$     & $\ocircle$     & $\CIRCLE$ \\
    \cmidrule{2-8}          & DyGPrompt~\cite{yu2025nodetime} & $\CIRCLE$     & $\ocircle$     & $\RIGHTcircle$     & $\CIRCLE$     & $\ocircle$     & $\CIRCLE$ \\
    \cmidrule{2-8}          & \modelname~\textbf{(ours)} & $\CIRCLE$     & $\CIRCLE$     & $\CIRCLE$    & $\CIRCLE$     & $\CIRCLE$     & $\CIRCLE$ \\
    \bottomrule
    \end{tabular}%
}
\label{tab:compare}
\end{table*}

%% file: 3_notation.tex
\section{Notations and Preliminaries}
\label{sec:notation}

\subsection{Notations}
Given a continuous-time dynamic graph $\mathcal{G} = \{\mathcal{V}, \mathcal{E}, \mathcal{T}\}$, where $\mathcal{V}$ denotes the node set, $\mathcal{T}$ denotes the continuous time domain, and $\mathcal{E}=\{(u,v,t)\}$ denotes the temporal edge set. Each temporal edge $(u,v,t)\in\mathcal{E}$ indicates that node $u$ interacts with node $v$ at timestamp $t\in\mathcal{T}$. The node feature matrix is denoted by $\mathbf{X}\in\mathbb{R}^{|\mathcal{V}|\times d_0}$, where $d_0$ is the input feature dimension. Let $\mathbf{z}_v^t \in \mathbb{R}^{d}$ denote the embedding of node $v$ at time $t$, where $d$ is the hidden dimension.

Dynamic Graph Neural Networks (DGNNs) serve as the backbone architecture for dynamic graph foundation modeling. In general, DGNNs update node embeddings by aggregating messages from historical temporal neighbors. Specifically, for node $v$ at time $t$, the aggregation is performed over previous interactions $(u,v,t^\prime)\in\mathcal{E}$ with $t^\prime<t$:
\begin{equation}
    \!\!\mathbf{z}_v^t\!=\!\operatorname{AGG}\!\left(
    \!\left\{
    \!\phi\!\left(
    \mathbf{z}_u^{t^\prime},
    \boldsymbol{\mathcal T}(t\!-\!t^\prime),
    \mathbf{e}_{uv}
    \!\right)
    \!\mid\! (u,v,t^\prime)\!\in\!\mathcal{E},\ t^\prime\!<\!t
    \right\}\!
    \right),\!\!
\end{equation}
where $\mathbf{z}_u^{t^\prime}$ denotes the historical embedding of node $u$ at time $t^\prime$, $\boldsymbol{\mathcal T}(\cdot)$ is a relative time encoder that maps the time interval $t-t^\prime$ into a temporal representation, $\mathbf{e}_{uv}$ denotes the edge feature between $u$ and $v$ when available, $\phi(\cdot)$ is a temporal message function, and $\operatorname{AGG}(\cdot)$ is a temporal neighborhood aggregator, \eg, the attention-based aggregator in TGAT~\cite{xu2020inductive}.

\subsection{Pre-training and Few-shot Fine-tuning}
Given $n$ dynamic graphs $\left\{\mathcal{G}_i^{\mathcal{S}}\right\}_{i=1}^n$ from source domains $\left\{D_i^{\mathcal{S}}\right\}_{i=1}^n$ with their corresponding labels $\left\{Y_i^{\mathcal{S}}\right\}_{i=1}^n$, the pre-training stage aims to learn a dynamic graph learner $h=g\circ f$, where $f$ denotes the dynamic graph encoder that extracts time-aware node representations, and $g$ denotes the task head. After convergence, the learned parameters $\boldsymbol{\theta}^{\star}$ are frozen and reused for downstream adaptation.

During fine-tuning, given a target dynamic graph $\mathcal{G}^{\mathcal{T}}$ from a target domain $D^{\mathcal{T}}$ (seen or unseen), only $m$ labeled samples are available as the support set under the $m$-shot setting, where $m\ll \sum_{i=1}^{n}|\mathcal{V}_i|$. The goal is to adapt the frozen $f$ to the target domain and predict labels for the unlabeled query set. In this work, adaptation is achieved by injecting learnable prompts into $f$, while keeping $\boldsymbol{\theta}^{\star}$ unchanged.

%% file: 4_method.tex
\begin{figure*}[!t]
    \centering
        \begin{minipage}{\textwidth}
            \hspace{-0.2cm}
            \begin{adjustbox}{width=1.02\linewidth}
            \input{fig/framework_marked}
            \end{adjustbox}
        \end{minipage}
    \vspace{-0.3cm}
    \caption{The framework of \modelname. \textbf{(1) Dual-Branch Pre-training.} \modelname~decouples transferable semantics and domain-specific temporal dynamics. \textbf{(2) Cross-Domain Adaptation.} Divergence-aware routing selects relevant source-domain experts according to semantic and temporal discrepancies. \textbf{(3) Conditioned Fine-tuning.} Divergence-conditioned graph prompts enable efficient target-domain adaptation with the frozen pre-trained encoder.}
\label{fig:framework}
\vspace{-0.5em}
\end{figure*}
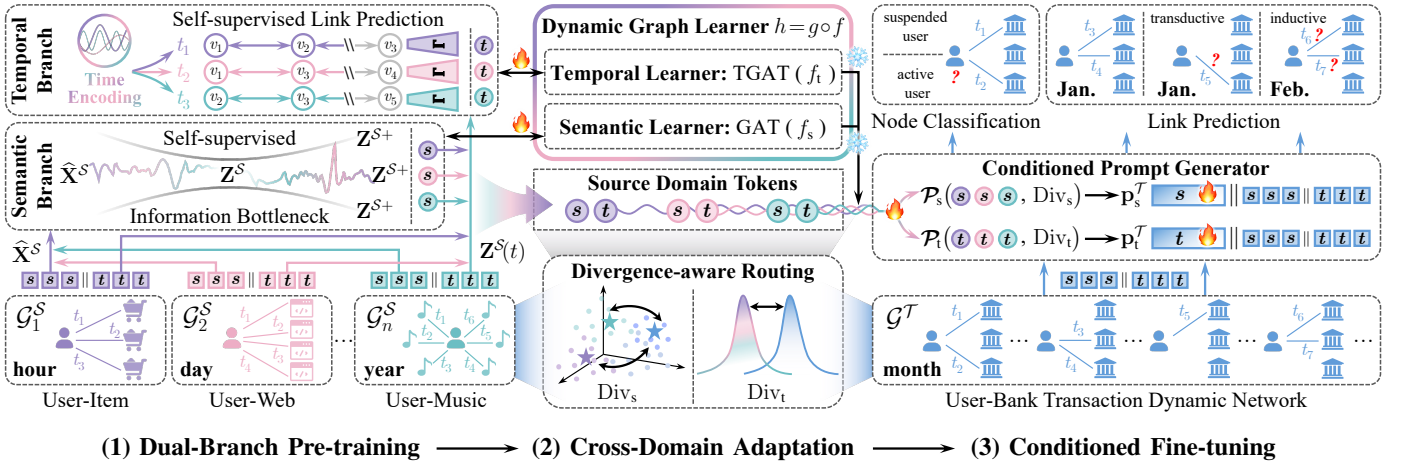

\section{Method}
\label{sec:method}
In this section, we elaborate on the proposed \modelname.

\subsection{Pre-training with Semantic-Temporal Decoupling}
\label{sec:pretrain}

Multi-domain dynamic graphs contain two entangled factors: transferable domain semantics and domain-specific temporal dynamics. Across domains, features may have different dimensions, distributions, and semantics, while temporal interactions may evolve at incompatible time scales. Directly merging such heterogeneous dynamic graphs for pre-training may therefore introduce feature-level semantic noise and temporal semantic distortion. To address this issue, we propose semantic-temporal decoupled pre-training. The semantic branch learns transferable semantic representations from timestamp-removed graphs, while the temporal branch preserves domain-specific dynamics through relative-time modeling and lightweight adapters.

\subsubsection{\textbf{Semantic Branch}}

The semantic branch aims to learn transferable domain semantics. Since source graphs may have heterogeneous feature spaces and domain-specific noise, directly applying a shared encoder to raw node features is ineffective. Therefore, we first project each source domain into a shared latent semantic space, and then constrain the learned representations to preserve informative semantics while suppressing redundant variations. For a source dynamic graph $\mathcal{G}_i^\mathcal{S}$, we remove timestamps and merge edges into a static graph with feature matrix $\mathbf{X}_i^\mathcal{S}\in\mathbb{R}^{|\mathcal{V}_i|\times d_i}$. Removing timestamps allows the semantic branch to focus on features and structures without being affected by incompatible temporal scales.

\textbf{Self-supervised Feature Alignment.}
Instead of aligning dimensions with trivial techniques like SVD~\cite{stewart1993early}, we propose a set of domain-specific aligners $\{\boldsymbol{\mathcal{A}}_i\}_{i=1}^n$ corresponding to each source domain. These aligners project heterogeneous node features into a shared latent semantic space:
\begin{equation}
    \widehat{\mathbf{X}}_i^\mathcal{S} \in\mathbb{R}^{|\mathcal{V}_i|\times d} = \boldsymbol{\mathcal{A}}_i\left(\mathbf{X}_i^\mathcal{S}\right),\quad \text{for each}~i \leqslant n,
\label{eq:aligner}
\end{equation}
where the aligner is implemented by a two-layer MLP~\cite{rosenblatt1958perceptron}.
Given the aligned node features, our goal is to learn compressed yet expressive semantic tokens $\mathbf{Z}^\mathcal{S}$ that retain transferable semantics while filtering domain-specific noise. Following the Information Bottleneck (IB) principle~\cite{tishby2000information, tishby2015deep}, we formulate a self-supervised objective that enforces both sufficiency, which preserves useful semantic information, and minimality, which filters out redundant variations:
\begin{align}
    \mathcal{L}_\text{SS-IB} &= -I\left(\mathbf{Z}_i^\mathcal{S};\mathbf{Z}_i^{\mathcal{S}+}\right) + \beta I\left(\mathbf{Z}_i^\mathcal{S}; \widehat{\mathbf{X}}_i^\mathcal{S}\right),\label{eq:ssib}\\
    \mathbf{Z}_i^\mathcal{S}&=f_\text{s}\left(\widehat{\mathbf{X}}_i^\mathcal{S}, \mathbf{A}_i^\mathcal{S}\right),\label{eq:ssib_next}
\end{align}
where $f_\text{s}$ is the GNN encoder. The first term (prediction term) maximizes the mutual information between $\mathbf{Z}_i^\mathcal{S}$ and its positive sample $\mathbf{Z}_i^{\mathcal{S}+}$ drawn from its neighbors. The second term (compression term) penalizes the mutual information between $\mathbf{Z}_i^\mathcal{S}$ and the aligned input $\widehat{\mathbf{X}}_i^\mathcal{S}$.

Since the two mutual information terms in Eq.~\eqref{eq:ssib} cannot be directly computed in closed form, we instantiate the self-supervised IB objective through two tractable approximations. Specifically, the prediction term is lower-bounded by a contrastive objective, while the compression term is upper-bounded by a variational KL regularizer.

\begin{prop}[Lower Bound of $I\left(\mathbf{Z}_i^\mathcal{S};\mathbf{Z}_i^{\mathcal{S}+}\right)$]
\label{prop:lower_bound}
    The prediction term $I\left(\mathbf{Z}_i^\mathcal{S};\mathbf{Z}_i^{\mathcal{S}+}\right)$ is intractable but can be lower-bounded via InfoNCE~\cite{oord2018representation}. For anchor node $u$ and its positive neighbor $u^+$ with negatives $\{v\}$, we have:
    \begin{align}
        I\left(\mathbf{Z}_i^\mathcal{S};\mathbf{Z}_i^{\mathcal{S}+}\right) &\geqslant -\mathcal{L}_\text{InfoNCE}\notag \\ 
        &=
        \frac{1}{N^+}
        \sum_{u=1}^{N^+}
        \log
        \frac{
        \exp\left(\left\langle \mathbf{z}_u^\mathcal{S}, \mathbf{z}_u^{\mathcal{S}+} \right\rangle/\tau\right)
        }{
        \sum_{v=0}^{N^-}
        \exp\left(\left\langle \mathbf{z}_u^\mathcal{S}, \mathbf{z}_v^{\mathcal{S}} \right\rangle/\tau\right)
        },
    \end{align}
    where $\langle \cdot,\cdot \rangle$ denotes the inner product (cosine similarity), $\tau$ is a temperature parameter, and negatives are sampled across domains to avoid trivial alignment.
\end{prop}

\begin{proof}
    Fix a source domain $D_i$. Define random variables for an anchor-positive pair sampled from domain $D_i$:
    \begin{equation}
        \left(\mathbf{Z}_i^{\mathcal S},\mathbf{Z}_i^{\mathcal S+}\right)\sim p_i\left(\mathbf{z},\mathbf{z}^{+}\right),
    \end{equation}
    where drawing an anchor node $u$ induces $\mathbf{z}=\mathbf{z}_u^{\mathcal S}$ and a positive $u^{+}$ induces $\mathbf{z}^{+}=\mathbf{z}_{u}^{\mathcal S+}$. The mutual information is:
    \begin{equation}
        I\left(\mathbf{Z}_i^{\mathcal S};\mathbf{Z}_i^{\mathcal S+}\right) =
        \mathbb{E}_{p_i\left(\mathbf{z},\mathbf{z}^{+}\right)}
        \left[
        \log\frac{p_i\left(\mathbf{z}\mid\mathbf{z}^{+}\right)}{p_i\left(\mathbf{z}\right)}
        \right].
    \end{equation}
    For each $u$ $\left(N^{+}\right)$ anchors in a batch, form a candidate set:
    \begin{equation}
        \mathcal{C}_u=
        \left\{
        \mathbf{z}_{0}^{\mathcal S},\cdots,\mathbf{z}_{N^-}^{\mathcal S}
        \right\},
        ~\mathbf{z}_{0}^{\mathcal S}=\mathbf{z}_{u}^{\mathcal S+},
        ~\mathbf{z}_{v}^{\mathcal S}\sim p_{\text{neg}}\left(\mathbf{z}\right),
    \end{equation}
    where $p_{\text{neg}}$ is the marginal distribution over embeddings used for negative sampling. It can be the within-domain marginal $p_i\left(\mathbf{Z}_i^{\mathcal S}\right)$ or a pooled marginal across domains. Both approximate $p\left(\mathbf{z}\right)$ and keep the derivation valid. Let $K=N^-+1$ be the candidate count.
    Define the softmax posterior over the index $v\in\left\{0,\cdots,N^-\right\}$ of the positive within $\mathcal C_u$:
    \begin{equation}
        q\left(v \mid\mathbf{z}_u^{\mathcal S},\mathcal C_u\right)=
        \frac{
        \exp\left(
        \left\langle \mathbf{z}_u^{\mathcal S},\mathbf{z}_{v}^{\mathcal S}\right\rangle/\tau
        \right)
        }{
        \sum_{k=0}^{N^-}
        \exp\left(
        \left\langle \mathbf{z}_u^{\mathcal S},\mathbf{z}_{k}^{\mathcal S}\right\rangle/\tau
        \right)
        }.
    \end{equation}
    Let $v^\star=0$ be the true index of the positive. The standard InfoNCE argument, namely contrastive predictive coding, gives:
    \begin{equation}
    \label{eq:temp1}
        \mathbb{E}\left[
        \log p\left(v^\star=0\mid \mathbf{z}_u^{\mathcal S},\mathcal C_u\right)
        \right]
        =
        I\left(\mathbf{Z}_i^{\mathcal S};\mathbf{Z}_i^{\mathcal S+}\right)-\log K,
    \end{equation}
    where the expectation is over the joint sampling of nodes for anchors, positives, and negatives. By the non-negativity of the KL divergence:
    \begin{align}
    \label{eq:temp2}
        &
        \mathbb{E}\left[
        \log q\left(\boldsymbol{0}\mid \mathbf{z}_u^{\mathcal S},\mathcal C_u\right)
        \right]\notag\\
        =
        ~&
        \mathbb{E}\left[
        \log p\left(\boldsymbol{0}\mid \mathbf{z}_u^{\mathcal S},\mathcal C_u\right)
        \right]
        -
        \operatorname{KL}\left(
        p\left(\cdot\mid \mathbf{z}_u^{\mathcal S},\mathcal C_u\right)
        ~\middle\|~
        q\left(\cdot\mid \mathbf{z}_u^{\mathcal S},\mathcal C_u\right)
        \right)\notag\\
        \leqslant
        ~&
        \mathbb{E}\left[
        \log p\left(\boldsymbol{0}\mid \mathbf{z}_u^{\mathcal S},\mathcal C_u\right)
        \right].
    \end{align}
    Averaging over anchors $u=1,\cdots,N^{+}$ yields:
    \begin{align}
        \!\!\!I\left(\mathbf{Z}_i^{\mathcal S};\mathbf{Z}_i^{\mathcal S+}\right)
        \geqslant
        &~\log K
        + \notag \\
        & ~\frac{1}{N^{+}}\!\!
        \sum_{u=1}^{N^{+}}
        \mathbb{E}\!
        \left[
        \log
        \frac{
        \exp\left(
        \left\langle \mathbf{z}_u^{\mathcal S},\mathbf{z}_{u}^{\mathcal S+}\right\rangle/\tau
        \right)
        }{
        \sum_{v=0}^{N^-}
        \exp\left(
        \left\langle \mathbf{z}_u^{\mathcal S},\mathbf{z}_{v}^{\mathcal S}\right\rangle/\tau
        \right)
        }
        \!\right]\!\!.\!\!\!
    \end{align}
    Replacing the expectation with the empirical batch average:
    \begin{equation}
        \mathcal{L}_{\text{InfoNCE}}
        =
        -
        \frac{1}{N^{+}}
        \sum_{u=1}^{N^{+}}
        \log
        \frac{
        \exp\left(
        \left\langle \mathbf{z}_u^{\mathcal S},\mathbf{z}_{u}^{\mathcal S+}\right\rangle/\tau
        \right)
        }{
        \sum_{v=0}^{N^-}
        \exp\left(
        \left\langle \mathbf{z}_u^{\mathcal S},\mathbf{z}_{u,v}^{\mathcal S}\right\rangle/\tau
        \right)
        }.
    \end{equation}
    It follows that:
    \begin{equation}
        I\left(\mathbf{Z}_i^{\mathcal S};\mathbf{Z}_i^{\mathcal S+}\right)
        \geqslant
        \log\left(N^-+1\right)-\mathcal{L}_{\text{InfoNCE}}
        \geqslant
        -\mathcal{L}_{\text{InfoNCE}},
    \end{equation}
    where the additive $\log\left(N^-+1\right)$ is a batch-size constant that does not affect optimization. We conclude the proof.
\end{proof}

Proposition~\ref{prop:lower_bound} makes the sufficiency constraint tractable by replacing the mutual information maximization with an InfoNCE-based contrastive objective. This encourages semantic tokens to retain neighborhood-consistent predictive information, while cross-domain negative samples prevent the representations from collapsing into trivial domain-specific alignments.

\begin{prop}[Upper Bound of $I\left(\mathbf{Z}_i^\mathcal{S}; \widehat{\mathbf{X}}_i^\mathcal{S}\right)$]
\label{prop:upper_bound}
     The compression term $I\left(\mathbf{Z}_i^\mathcal{S}; \widehat{\mathbf{X}}_i^\mathcal{S}\right)$ is intractable but can be upper-bounded by KL divergence~\cite{cover1999elements} between the posterior $q_\phi\left(\mathbf{z}^\mathcal{S}\mid \widehat{\mathbf{x}}^\mathcal{S}\right)$ and prior $p\left(\mathbf{z}^\mathcal{S}\right)$:
    \begin{equation}
        I\left(\mathbf{Z}_i^\mathcal{S}; \widehat{\mathbf{X}}_i^\mathcal{S}\right) \leqslant
        \frac{1}{N}\sum_{u=1}^N
        \operatorname{KL}\left(
        q_\phi\left(\mathbf{z}_u^\mathcal{S}\mid \widehat{\mathbf{x}}_u^\mathcal{S}\right)
        ~\middle\|~
        p\left(\mathbf{z}_u^\mathcal{S}\right)
        \right).
    \end{equation}
\end{prop}

\begin{proof}
    The compression term penalizes how much information the latent representation $\mathbf{Z}_i^{\mathcal S}$ preserves about the input $\widehat{\mathbf{X}}_i^{\mathcal S}$. Formally, the mutual information can be written as:
    \begin{equation}
        I\left(\mathbf{Z}_i^{\mathcal S}; \widehat{\mathbf{X}}_i^{\mathcal S}\right)
        =
        \mathbb{E}_{p\left(\widehat{\mathbf{x}},\mathbf{z}\right)}
        \left[
        \log
        \frac{
        p\left(\mathbf{z}\mid \widehat{\mathbf{x}}\right)
        }{
        p\left(\mathbf{z}\right)
        }
        \right].
    \end{equation}
    This is equivalent to:
    \begin{equation}
        I\left(\mathbf{Z}_i^{\mathcal S}; \widehat{\mathbf{X}}_i^{\mathcal S}\right)
        =
        \mathbb{E}_{p\left(\widehat{\mathbf{x}}\right)}
        \left[
        \operatorname{KL}\left(
        p\left(\mathbf{z}\mid \widehat{\mathbf{x}}\right)
        ~\middle\|~
        p\left(\mathbf{z}\right)
        \right)
        \right],
    \end{equation}
    which indicates that compression measures the average discrepancy between the posterior and the prior.
    The true posterior $p\left(\mathbf{z}\mid \widehat{\mathbf{x}}\right)$ is generally intractable. We approximate it with a variational distribution $q_\phi\left(\mathbf{z}\mid \widehat{\mathbf{x}}\right)$. By the non-negativity of the KL divergence:
    \begin{equation}
        \operatorname{KL}\left(
        p\left(\mathbf{z}\mid \widehat{\mathbf{x}}\right)
        ~\middle\|~
        p\left(\mathbf{z}\right)
        \right)
        \leqslant
        \operatorname{KL}\left(
        q_\phi\left(\mathbf{z}\mid \widehat{\mathbf{x}}\right)
        ~\middle\|~
        p\left(\mathbf{z}\right)
        \right).
    \end{equation}
    Taking expectation over aligned inputs and replacing it with an empirical average over nodes $u=1,\cdots, N$, we obtain:
    \begin{equation}
        I\left(\mathbf{Z}_i^{\mathcal S}; \widehat{\mathbf{X}}_i^{\mathcal S}\right)
        \leqslant
        \frac{1}{N}
        \sum_{u=1}^{N}
        \operatorname{KL}\left(
        q_\phi\left(\mathbf{z}_u^{\mathcal S}\mid \widehat{\mathbf{x}}_u^{\mathcal S}\right)
        ~\middle\|~
        p\left(\mathbf{z}_u^{\mathcal S}\right)
        \right).
    \end{equation}
    We conclude the proof.
\end{proof}

Proposition~\ref{prop:upper_bound} constrains semantic representations with a variational KL regularizer, encouraging them to remain close to a prior distribution and thereby reducing redundant domain-specific variations for more stable cross-domain pre-training.

\textbf{Pre-training Objectives.}
After obtaining semantic tokens, we optimize the semantic branch using a self-supervised link-based objective on the timestamp-removed graph. This objective encourages connected nodes to have consistent semantic representations, while separating them from unconnected nodes across the merged training graphs. Inspired by prior studies~\cite{liu2023graphprompt}, each anchor node $u$ is paired with a positive neighbor $v^+$ sampled from its observed adjacency and a negative neighbor $v^-$ sampled from unconnected nodes. Concretely, a similarity discriminator $g=\operatorname{MLP}\left(\langle\cdot,\cdot\rangle\right)$ maps a pair of embeddings to a scalar score, where $\langle\cdot,\cdot\rangle$ denotes the inner product. For each non-redundant training quadruple $(u, v^+, v^-, \mathbf{y}_u)$, where $\mathbf{y}_u\in\{0,1\}$ indicates link existence, the objective is:
\begin{align}
    \!\!\!&~\mathcal{L}_\text{pre}\left(\boldsymbol{\theta}_\text{s},\mathbf{Z}^{\mathcal S}\right)\notag\\
    =
    &-\!\!\!\!\!\!\!
    \sum_{(u, v^+, v^-)}
    \!\!\!\!\!\!\log
    \frac{
    \exp\left(g\left(\mathbf{z}_u^{\mathcal S}, \mathbf{z}_{v^+}^{\mathcal S}\right) / \tau\right)
    }{
    \exp\left(g\!\left(\mathbf{z}_u^{\mathcal S}, \mathbf{z}_{v^+}^{\mathcal S}\right) \!/ \tau\right)
    \!+\!
    \exp\left(g\!\left(\mathbf{z}_u^{\mathcal S}, \mathbf{z}_{v^-}^{\mathcal S}\right)\!/ \tau\right)
    }.\!\!\!
\label{eq:pre_sem}
\end{align}
Finally, we combine the link-based pre-training loss with the self-supervised IB regularization. The overall objective for the semantic branch is:
\begin{equation}
    \mathcal{L}_\text{pre-sem}\left(\boldsymbol{\theta}_\text{s},\mathbf{Z}^{\mathcal S}\right)
    =
    \mathcal{L}_\text{pre}\left(\boldsymbol{\theta}_\text{s},\mathbf{Z}^{\mathcal S}\right)
    +
    \lambda_1 \cdot \mathcal{L}_\text{SS-IB}.
\label{eq:pre_sem_overall}
\end{equation}
Here, $\lambda_1$ controls the strength of the information bottleneck regularization, balancing semantic discrimination and redundancy suppression during pre-training.

\subsubsection{\textbf{Temporal Branch}}
While the semantic branch learns transferable feature semantics in a shared latent space, temporal dynamics require a different treatment. The timestamps of different domains usually do not share the same physical semantics: a short interval may indicate an important behavioral change in transaction networks, but may be negligible in citation or collaboration networks. Therefore, directly sharing a single time encoder across all domains may force incompatible temporal scales into the same representation space and distort domain-specific evolution patterns.

To preserve temporal specificity while still learning reusable temporal aggregation knowledge, the temporal branch follows a shared-specific design. Specifically, domain-specific timers encode relative time within each source domain, a shared temporal backbone captures general temporal dependency patterns, and lightweight adapters absorb domain-specific residual dynamics.

\textbf{Relative-time Encoding.}
For each edge $(u,v,t)\in\mathcal E_i$, a domain-specific timer $\boldsymbol{\mathcal{T}}_i$ maps $t$ to a time embedding:
\begin{equation}
    \mathbf{r}^{\mathcal S}_i\left(t\right)
    =
    \boldsymbol{\mathcal{T}}_i\left(t\right),
    \quad
    \boldsymbol{\mathcal{T}}_i
    =
    \operatorname{MLP}_i\left(\cdot\right)
    \circ
    \operatorname{PE}\left(\cdot\right),
\label{eq:rte}
\end{equation}
where $\operatorname{PE}\left(\cdot\right)$ is a sinusoidal positional encoding similar to that used in Transformers~\cite{vaswani2017attention}, which maps the relative timestamp into a vector as:
\begin{align}
    \!\!\!\!\!\!&\operatorname{PE}\left(t\right) \notag\\
    =
    &\left[
    \sin\left(\omega_1 t\right),
    \cos\left(\omega_1 t\right),
    \cdots\!,
    \sin\left(\omega_{{d}_t/2} t\right),
    \cos\left(\omega_{{d}_t/2} t\right)
    \right]\!,\!
\end{align}
where $\omega_i$ is the frequency, and $d_t$ is the time dimension. Since $\boldsymbol{\mathcal{T}}_i$ is domain-specific, each source domain can preserve its own temporal resolution and evolution pattern without being forced to share a globally aligned time scale.

\textbf{Shared Backbone and Domain Adapters.}
After relative-time encoding, the remaining challenge is how to share temporal knowledge without erasing domain-specific dynamics. A fully shared temporal model may overlook domain-specific rhythms, whereas fully independent temporal models would prevent cross-domain knowledge transfer. To balance universality and specificity, we adopt a shared TGAT~\cite{xu2020inductive} backbone $f_\text{t}$ to model general temporal aggregation patterns, and complement it with lightweight residual adapters $\{\boldsymbol{\Gamma}_i\}_{i=1}^n$, implemented as bottleneck MLPs, to capture domain-specific temporal residuals.
Denote neighbors as $\mathcal N_i\left(v,t\right)=\left\{(u,t^\prime)\in\mathcal E_i: t^\prime<t\right\}$. Messages $\mathbf{m}$ and time-aware attention $a$ are:
\begin{align}
    &\mathbf{m}_{u\to v}^{\mathcal S}\left(t^\prime\right)
    =
    \phi\left(
    \mathbf{z}^{\mathcal S}_u\left(t^\prime\right),
    \mathbf{r}_i^{\mathcal S}\left(t-t^\prime\right)
    \right),\\
    &a_{u\to v}\left(t\right)
    =
    \underset{(u,t^\prime)\in\mathcal N_i\left(v,t\right)}{\operatorname{Softmax}}
    \psi\left(
    \mathbf{z}^{\mathcal S}_v\left(<t\right),
    \mathbf{m}^{\mathcal S}_{u\to v}\left(t^\prime\right)
    \right),
\end{align}
where $\psi\left(\cdot\right)$ is a ranking function, and we implement it by the dot product. The backbone state is:
\begin{equation}
    \mathbf{z}^{\mathcal S}_v\left(t\right)
    =
    \operatorname{AGG}\Big(
    \sum_{(u,t^\prime)\in\mathcal N_i\left(v,t\right)}
    a_{u\to v}\left(t\right)
    \mathbf{m}^{\mathcal S}_{u\to v}\left(t^\prime\right)
    \Big).
\end{equation}
Formally, given $\mathbf{z}_v^{\mathcal S}\left(t\right)$, the domain adapter refines it as:
\begin{equation}
    \widetilde{\mathbf{z}}^{\mathcal S}_v\left(t\right)
    =
    \mathbf{z}^{\mathcal S}_v\left(t\right)
    +
    \boldsymbol{\Gamma}_i\left(
    \mathbf{z}^{\mathcal S}_v\left(t\right)
    \right).
\label{eq:update}
\end{equation}
In this way, the shared backbone learns reusable temporal aggregation knowledge across domains, while the adapters compensate for domain-specific temporal dynamics that cannot be captured by a universal temporal model alone.

\textbf{Alternating Training and Freezing Strategy.}
Training the temporal branch across multiple domains requires careful separation between domain-general temporal aggregation and domain-specific temporal residuals. If adapters are updated too early, they may absorb patterns that should be learned by the shared backbone, weakening cross-domain transfer. Conversely, if all parameters are trained jointly without constraints, domain-specific timestamp semantics may interfere with the learning of general temporal dependencies. To avoid these issues, we adopt an alternating training and stage-wise freezing strategy.

Denote $\boldsymbol{\Theta}_\text{t}$ as the parameters of $f_\text{t}$, $\boldsymbol{\Phi}_i$ denote the parameters of timer for domain $D_i^{\mathcal S}$, and $\boldsymbol{\Psi}_i$ denote the parameters of the adapter $\boldsymbol{\Gamma}_i$. For mini-batch $\mathcal B_i \subset \mathcal G_i^{\mathcal S}$, the temporal pre-training loss is computed in the same link-based form as Eq.~\eqref{eq:pre_sem}:
\begin{align}
    &
    \mathcal{L}_\text{pre-tem}^{(i)}
    \left(
    \boldsymbol{\Theta}_\text{t},
    \boldsymbol{\Phi}_i,
    \boldsymbol{\Psi}_i;
    \mathcal{B}_i
    \right)
    \notag\\
    =
    &-\!\!\!\!\!\!\!
    \sum_{\left(u, v^+, v^-,t\right)\in\mathcal{B}_i}
    \!\!\!\!\!\!\!\log
    \frac{
    \exp\left(
    g\left(
    \widetilde{\mathbf{z}}_u^{\mathcal S},
    \widetilde{\mathbf{z}}_{v^+}^{\mathcal S}
    \right) / \tau
    \right)
    }{
    \exp\!\left(
    g\!\left(
    \widetilde{\mathbf{z}}_u^{\mathcal S},
    \widetilde{\mathbf{z}}_{v^+}^{\mathcal S}
    \right)\! / \tau
    \right)
    \!+\!
    \exp\!\left(
    g\!\left(
    \widetilde{\mathbf{z}}_u^{\mathcal S},
    \widetilde{\mathbf{z}}_{v^-}^{\mathcal S}
    \right)\! / \tau
    \right)
    }.
\label{eq:pre_loss_temp}
\end{align}

During Phase I (backbone-timer training), we update the shared temporal backbone and the domain-specific timers, while freezing the adapters. This phase encourages $f_\text{t}$ to learn reusable temporal aggregation patterns and allows each timer to preserve its domain-specific time scale:
\begin{equation}
    \!\!\!\left(
    \boldsymbol{\Theta}_\text{t},
    \boldsymbol{\Phi}_i
    \right)
    \!\gets\!
    \left(
    \boldsymbol{\Theta}_\text{t},
    \boldsymbol{\Phi}_i
    \right)
    \!-\!
    \eta
    \nabla_{
    \left(
    \boldsymbol{\Theta}_\text{t},
    \boldsymbol{\Phi}_i
    \right)}
    \mathcal{L}_\text{pre-tem}^{(i)}
    \!\left(
    \boldsymbol{\Theta}_\text{t},
    \boldsymbol{\Phi}_i,
    \boldsymbol{\Psi}_i;
    \mathcal{B}_i
    \right),
\end{equation}
while $\nabla_{\boldsymbol{\Psi}_i}=\boldsymbol{0}$, and $\eta$ denotes the learning rate. Mini-batches $\{\mathcal B_i\}_{i=1}^n$ are sampled in turn, with each mini-batch containing the same number of samples to avoid domain-size bias.

After Phase I converges, Phase II freezes the shared backbone and timers, and only updates the adapters:
\begin{equation}
    \boldsymbol{\Psi}_i
    \gets
    \boldsymbol{\Psi}_i
    -
    \eta
    \nabla_{\boldsymbol{\Psi}_i}
    \mathcal{L}_\text{pre-tem}^{(i)}
    \left(
    \boldsymbol{\Psi}_i;
    \mathcal{B}_i
    \right),
\end{equation}
while keeping $\nabla_{\boldsymbol{\Theta}_\text{t}}=\boldsymbol{0}$ and $\nabla_{\boldsymbol{\Phi}_i}=\boldsymbol{0}$. Each adapter is updated only when its corresponding domain is sampled, and this lets adapters capture domain residual dynamics without overwriting the shared temporal knowledge learned by the backbone.

\subsection{Adaptation with Cross-Domain Negative Transfer Mitigation}
\label{sec:moe}

After semantic-temporal decoupled pre-training, \modelname~obtains two types of source-domain knowledge: semantic tokens from $f_\text{s}$ and temporal tokens from $f_\text{t}$, together with domain-specific timers $\{\boldsymbol{\mathcal{T}}_i\}_{i=1}^n$ and adapters $\{\boldsymbol{\Gamma}_i\}_{i=1}^n$, which provide reusable priors for downstream adaptation. However, not every source domain is equally useful for a given target domain. Directly mixing all source experts may introduce irrelevant semantic or temporal knowledge, while selecting only a single source expert may overlook complementary cross-domain patterns. Therefore, adaptation should selectively reuse source knowledge according to its relevance to the target domain.

To mitigate such negative transfer, we formulate cross-domain adaptation as a divergence-aware Mixture-of-Experts mechanism. The core idea is to measure how close each source domain is to the target domain from both semantic and temporal perspectives, and then route the target graph to the most relevant source-domain experts. In this way, \modelname~can balance transferable cross-domain knowledge with target-specific adaptation.

\textbf{Expert Pool and Prototypes.}
For each source domain $D_i^{\mathcal S}$, we summarize its semantic and temporal knowledge with two prototypes. The semantic prototype $\mathbf{s}_i^{\mathcal S}$ represents the average semantic token of the domain, while the temporal prototype $\mathbf{t}_i^{\mathcal S}$ summarizes its dynamic evolution pattern:
\begin{equation}
    \mathbf{s}_i^{\mathcal S}
    =
    \frac{1}{|\mathcal V_i^{\mathcal S}|}
    \sum_{v\in\mathcal V_i^{\mathcal S}}
    \mathbf{z}_{v}^{\mathcal S},
    \quad
    \mathbf{t}_i^{\mathcal S}
    =
    \frac{1}{|\mathcal E_i^{\mathcal S}|}
    \sum_{(v,t)\in\mathcal E_i^{\mathcal S}}
    \widetilde{\mathbf{z}}_{v}^{\mathcal S}\left(t\right).
\label{eq:prototype}
\end{equation}
Given the target dynamic graph $\mathcal G^{\mathcal T}$, we construct its target descriptors in the same semantic-temporal manner. For semantic description, we first build graph $\overline{\mathcal G}^{\mathcal T}=\big\{\overline{\mathcal V}^{\mathcal T}, \overline{\mathcal E}^{\mathcal T}\big\}$ by discarding timestamps. Let $\widehat{\mathbf X}^{\mathcal T}$ and $\widehat{\mathbf A}^{\mathcal T}$ denote its aligned feature matrix and adjacency matrix, where $\widehat{\mathbf X}^{\mathcal T}$ is dimension-aligned by Eq.~\eqref{eq:aligner}. The frozen semantic encoder $f_\text{s}^\star$ produces:
\begin{equation}
    \mathbf{s}^{\mathcal T}
    =
    \frac{1}{|\mathcal V^{\mathcal T}|}
    \sum\limits_{v\in \mathcal V^{\mathcal T}}
    \mathbf{z}_v^{\mathcal T},
    \quad
    \mathbf{Z}^{\mathcal T}
    =
    f_{\text{s}}^\star\left(
    \widehat{\mathbf X}^{\mathcal T},
    \widehat{\mathbf A}^{\mathcal T}
    \right).
\label{eq:s_descrip}
\end{equation}
For temporal description, we keep timestamps and aggregate dynamic embeddings with the frozen temporal encoder $f_\text{t}^\star$:
\begin{equation}
    \mathbf{t}^{\mathcal T}
    \!=\!
    \frac{1}{|\mathcal E^{\mathcal T}|}
    \sum_{(v,t)\in \mathcal E^{\mathcal T}}\!\!\!\!
    \widetilde{\mathbf{z}}^{\mathcal T}_v\!\left(t\right),
    ~
    \widetilde{\mathbf{z}}^{\mathcal T}_v\!\!\left(t\right)
    =
    f_{\text{t}}^\star\!\left(
    \widehat{\mathbf X}^{\mathcal T},
    \mathcal E^{\mathcal T},
    \mathbf{r}^{\mathcal T}\left(t\right)
    \right)\!,
\label{eq:t_descrip}
\end{equation}
where $\mathbf{r}^{\mathcal T}\left(t\right)$ is the relative time encoding defined as in Eq.~\eqref{eq:rte}. The pair $\left(\mathbf{s}^{\mathcal T}, \mathbf{t}^{\mathcal T}\right)$ serves as the target-domain descriptor for subsequent routing.

\textbf{Divergence-aware Routing.}
The routing mechanism should assign higher weights to the source domains that are semantically and temporally closer to the target domain. To this end, we compute two types of divergences between each source prototype $\left(\mathbf{s}_i^{\mathcal S}, \mathbf{t}_i^{\mathcal S}\right)$ and the target descriptor $\left(\mathbf{s}^{\mathcal T}, \mathbf{t}^{\mathcal T}\right)$.

\subsubsection{\textbf{Semantic-wise Divergence}}

For semantics, we parameterize source and target semantic representations as diagonal Gaussian distributions:
\begin{equation}
    \mathcal{N}\left(
    \boldsymbol{\mu}_{i}^{\mathcal S}\!=\!\mathbf{s}_i^{\mathcal S},
    \operatorname{diag}\left(\boldsymbol{\sigma}_{i}^{2\mathcal S}\right)
    \right),
    ~
    \mathcal N\left(
    \boldsymbol{\mu}^{\mathcal T}\!=\!\mathbf{s}^{\mathcal T},
    \operatorname{diag}\left(\boldsymbol{\sigma}^{2\mathcal T}\right)
    \right),
\end{equation}
with variances computed via sample scatter around prototypes:
\begin{align}
    &\boldsymbol{\sigma}_{i}^{2\mathcal S}
    =
    \frac{1}{|\mathcal V_i|}
    \sum_{v\in\mathcal V_i}
    \left(
    \mathbf{z}^{\mathcal S}_{v}
    -
    \mathbf{s}_i^{\mathcal S}
    \right)^{\odot 2},\\
    ~
    &\boldsymbol{\sigma}^{\mathcal T2}
    =
    \frac{1}{\left|\overline{\mathcal V}^{\mathcal T}\right|}
    \sum_{v\in\overline{\mathcal V}^{\mathcal T}}
    \left(
    \widetilde{\mathbf{z}}_{v}^{\mathcal T}
    -
    \mathbf{s}^{\mathcal T}
    \right)^{\odot 2},
\end{align}
where ``$\odot 2$'' denotes element-wise square. The semantic divergence is then defined as the closed-form KL divergence between diagonal Gaussian distributions:
\begin{align}
    &\operatorname{Div}_{\text{s}}\left(D^{\mathcal T},D_i^{\mathcal S}\right)
    \notag \\
    =
    &\frac{1}{2}
    \sum_{k=1}^{d}
    \Bigg(
    \log
    \frac{\boldsymbol{\sigma}_{i,k}^{2\mathcal S}}{\boldsymbol{\sigma}_{k}^{2\mathcal{T}}}
    +
    \frac{
    \boldsymbol{\sigma}_{k}^{2\mathcal{T}}
    +
    \left(
    \boldsymbol{\mu}_{k}^\mathcal T
    -
    \boldsymbol{\mu}_{i,k}^\mathcal{S}
    \right)^{2}
    }{
    \boldsymbol{\sigma}_{i,k}^{2\mathcal{S}}
    }
    -
    1
    \Bigg).
\label{eq:s_div}
\end{align}

\subsubsection{\textbf{Temporal-wise Divergence}}
For temporal dynamics, we measure the discrepancy between the target and each source temporal prototype with squared Euclidean distance:
\begin{equation}
    \operatorname{Div}_{\text{t}}\left(D^{\mathcal T},D_i^{\mathcal S}\right)
    =
    \left\|\mathbf{t}^{\mathcal T} - \mathbf{t}_i^{\mathcal S}\right\|_2^2.
\label{eq:t_div}
\end{equation}

The combined divergence is a weighted sum:
\begin{equation}
    \operatorname{Div}_i
    =
    \lambda_{\text{s}}\cdot
    \operatorname{Div}_{\text{s}}\left(D^{\mathcal T},D_i^{\mathcal S}\right)
    +
    \lambda_{\text{t}}\cdot
    \operatorname{Div}_{\text{t}}\left(D^{\mathcal T},D_i^{\mathcal S}\right),
\label{eq:div_comb}
\end{equation}
where $\lambda_{\text{s}}$ and $\lambda_{\text{t}}$ are trade-off parameters.

\textbf{Routing Optimization.}
Based on the combined divergence, the expert router produces mixture weights $\boldsymbol{\alpha}\in\mathbb{R}^n$ by applying a softmax over negative divergences. A source domain with smaller divergence receives a larger routing weight, indicating higher relevance to the target domain:
\begin{align}
    &\boldsymbol{\alpha}_i
    =
    \operatorname{Softmax}_i^n\left(-\gamma\operatorname{Div}_i\right),
    \label{eq:routing} \\
    &\bar{\mathbf{s}}^\mathcal{S}
    =
    \sum\nolimits_{i=1}^n
    \boldsymbol{\alpha}_i \mathbf{s}_i^\mathcal{S},
    \quad
    \bar{\mathbf{t}}^\mathcal{S}
    =
    \sum\nolimits_{i=1}^n
    \boldsymbol{\alpha}_i \mathbf{t}_i^\mathcal{S},
\end{align}
where $\gamma$ is a temperature parameter. The prototypes $\bar{\mathbf{s}}^\mathcal{S}$ and $\bar{\mathbf{t}}^\mathcal{S}$ serve as cross-domain priors for target-domain fine-tuning.

To further stabilize routing, we regularize the weights as:
\begin{equation}
    \mathcal R_{\text{routing}}\left(\boldsymbol{\alpha}\right)
    =
    \lambda_2\cdot
    \sum\nolimits_{i=1}^n
    \boldsymbol{\alpha}_i \,\mathrm{Div}_i
    -
    \lambda_3\cdot
    H\left(\boldsymbol{\alpha}\right),
\label{eq:routing_r}
\end{equation}
where $H\left(\cdot\right)$ denotes the Shannon entropy~\cite{shannon1948mathematical}. The first term suppresses high-divergence domains to reduce negative transfer, while the entropy term prevents the router from collapsing to a single expert too early. This regularized routing design enables \modelname~to selectively reuse relevant source-domain knowledge while maintaining sufficient expert diversity.

\subsection{Fine-tuning with Divergence-Conditioned Prompting}
\label{sec:finetune}

After divergence-aware routing, \modelname~obtains source-domain priors that are relevant to the target domain. However, directly updating the pre-trained encoders for each downstream task would be inefficient and may damage the transferable knowledge learned during pre-training. Therefore, we adapt the frozen encoders through lightweight prompts. The key idea is to condition prompt generation on both target-domain descriptors and source-target divergences, so that the prompts can adjust the frozen semantic and temporal encoders according to the target domain while preserving pre-trained knowledge.

\textbf{Conditioned Prompt Generator.}
Semantic and temporal discrepancies provide informative signals for target-domain adaptation. A larger divergence indicates that the target domain deviates more from the source-domain priors, and therefore requires stronger prompt-based adjustment. To this end, we introduce two parallel prompt generators $\left(\boldsymbol{\mathcal{P}}_\text{s}, \boldsymbol{\mathcal{P}}_\text{t}\right)$ to produce semantic and temporal prompts $\left(\mathbf{p}_\text{s}^\mathcal{T},\mathbf{p}_\text{t}^\mathcal{T}\right)$:
\begin{align}
    &\mathbf{p}_\text{s}^\mathcal{T}
    =
    \boldsymbol{\mathcal{P}}_\text{s}
    \left(
    \mathbf{s}^{\mathcal T},
    \frac{1}{n}
    \sum_{i=1}^n
    \operatorname{Div}_{\text{s}}
    \left(D^{\mathcal T},D_i^{\mathcal S}\right)
    \right),\label{eq:prompt_1} \\
    &\mathbf{p}_\text{t}^\mathcal{T}
    =
    \boldsymbol{\mathcal{P}}_\text{t}
    \left(
    \mathbf{t}^{\mathcal T},
    \frac{1}{n}
    \sum_{i=1}^n
    \operatorname{Div}_{\text{t}}
    \left(D^{\mathcal T},D_i^{\mathcal S}\right)
    \right),\label{eq:prompt_2}
\end{align}
where the prompt generators $\boldsymbol{\mathcal{P}}_\text{s}$ and $\boldsymbol{\mathcal{P}}_\text{t}$ are implemented by MLPs with learnable parameters $\boldsymbol{\Omega}_\text{s}$ and $\boldsymbol{\Omega}_\text{t}$, respectively.

\textbf{Prompt Injection and Fusion.}
With all pre-trained encoders frozen, the generated learnable prompts are injected into the encoder inputs rather than used to update the backbone parameters. The semantic prompt adjusts the aligned target features for semantic encoding, while the temporal prompt adjusts the target features used by the temporal encoder:
\begin{align}
    &\mathbf{Z}^{\mathcal T}_{\mathbf{p}}
    =
    f_{\text{s}}^\star
    \left(
    \left[
    \mathbf{p}_\text{s}^\mathcal{T};
    \widehat{\mathbf X}^{\mathcal T}
    \right],
    \widehat{\mathbf A}^{\mathcal T}
    \right),
    \\
    &\widetilde{\mathbf{Z}}^{\mathcal T}_{\mathbf{p}}\left(t\right)
    =
    f_{\text{t}}^\star
    \left(
    \left[
    \mathbf{p}_\text{t}^\mathcal{T};
    \widehat{\mathbf X}^{\mathcal T}
    \right],
    \mathcal{E}^{\mathcal T},
    \mathbf{r}^\mathcal{T}\left(t\right)
    \right),
\label{eq:inject}
\end{align}
where $\left[\cdot~;~\cdot\right]$ denotes embedding concatenation. For node $v$ at time $t$, the prompted representation is obtained by fusing the semantic and temporal outputs:
\begin{equation}
    \!\!\mathbf{H}^{\mathcal T}\!\!\left(t\right)
    \!=\!
    \phi
    \left(
    \left[
    \mathbf{Z}^{\mathcal T}_{\mathbf{p}};
    \widetilde{\mathbf{Z}}^{\mathcal T}_{\mathbf{p}}\left(t\right)
    \right]
    \right),
    ~
    \phi\left(\cdot\right)
    \!=\!
    \text{LN}\left(\cdot\right)
    \circ
    \text{Linear}\left(\cdot\right),\!\!
\label{eq:fuse}
\end{equation}
where $\text{LN}\left(\cdot\right)$ denotes the LayerNorm operation.

\textbf{Task Objectives.}
To maintain consistency with pre-training, we use the similarity-based task objectives for downstream fine-tuning. This design aligns downstream prediction with the similarity template used in pre-training, so that prompt tuning can adapt the target task without changing the frozen encoders.

\subsubsection{\textbf{Temporal Node Classification}}
Given a labeled support set $\mathcal D_{\text{sup}}^\mathcal{T}=\{(v_i,\mathbf{y}_i,t_i)\}$, define the time-aware class prototype as the mean of support embeddings of class $\mathbf{y}$ at time $t$:
\begin{equation}
    \overline{\mathbf{h}}_{t,\mathbf{y}}^{\mathcal T}
    =
    \frac{1}{
    \left|
    \mathcal D_{\text{sup}}^\mathcal{T}\left(t,\mathbf{y}\right)
    \right|
    }
    \sum\nolimits_{\left(v_i,\mathbf{y}_i=\mathbf{y},t_i=t\right)\in \mathcal D_{\text{sup}}^\mathcal{T}\left(t,\mathbf{y}\right)}
    \mathbf{h}_{v_i}^{\mathcal T}\left(t\right).
\end{equation}
The loss for query $\left(v_i,\mathbf{y}_i,t_i\right)$ is defined by a prototype-softmax:
\begin{align}
    &\mathcal L_{\text{node}}
    \left(
    \boldsymbol{\Omega}_\text{s},
    \boldsymbol{\Omega}_\text{t}
    \right) \notag \\
    =
    &-\!\!\!\!
    \sum_{\left(v_i,\mathbf{y}_i,t_i\right)\in\mathcal D_{\text{sup}}^\mathcal{T}}
    \!\!\!\!\log
    \frac{
    \exp\left(
    g\left(
    \mathbf{h}^{\mathcal T}_{v_i}\left(t_i\right),
    \overline{\mathbf{h}}^{\mathcal T}_{t_i,\mathbf{y}_i}
    \right)/\tau
    \right)
    }{
    \sum\nolimits_{\mathbf{y}\in\mathcal Y}
    \exp\left(
    g\left(
    \mathbf{h}^{\mathcal T}_{v_i}\left(t_i\right),
    \overline{\mathbf{h}}^{\mathcal T}_{t_i,\mathbf{y}}
    \right)/\tau
    \right)
    }.
\label{eq:node}
\end{align}

\subsubsection{\textbf{Temporal Link Prediction}}
Given positive edge $\left(u,v^{+},t\right)$ and negatives $\{(u,v^{-},t)\}$, we apply an InfoNCE~\cite{oord2018representation} loss:
\begin{align}
    \!\!\!&
    \mathcal L_{\text{link}}
    \left(
    \boldsymbol{\Omega}_\text{s},
    \boldsymbol{\Omega}_\text{t}
    \right)
    \notag\\
    =
    &-\!\!\!\!\!\!\!
    \sum_{\left(u,v^{+},v^{-},t\right)}
    \!\!\!\!\!\!\!\log
    \frac{
    \exp\left(
    g\left(
    \mathbf{h}_u^\mathcal{T},
    \mathbf{h}_{v^{+}}^\mathcal{T}
    \right)/\tau
    \right)
    }{
    \exp\!\left(
    g\!\left(
    \mathbf{h}_u^\mathcal{T},
    \mathbf{h}_{v^{+}}^\mathcal{T}
    \right)\!/\tau
    \right)
    \!+\!
    \sum\limits_{v^{-}}
    \exp\!\left(
    g\!\left(
    \mathbf{h}_u^\mathcal{T},
    \mathbf{h}_{v^{-}}^\mathcal{T}
    \right)\!/\tau
    \right)
    }.\!\!\!
\label{eq:link}
\end{align}

\textbf{Overall Fine-tuning Objective.}
The fine-tuning objective combines downstream task loss with routing regularizer:
\begin{align}
    &\mathcal L_{\text{ftn}}
    \left(
    \boldsymbol{\Omega}_\text{s},
    \boldsymbol{\Omega}_\text{t},
    \boldsymbol{\alpha}
    \right)
    =
    \mathcal L_{\text{task}}
    +
    \lambda_r\cdot
    \mathcal R_{\text{routing}}
    \left(
    \boldsymbol{\alpha}
    \right), \label{eq:ftn_overall}\\
    &\text{task}\in
    \left\{
    \text{node},
    \text{link}
    \right\},
\end{align}
where $\lambda_r$ is a hyper-parameter. The task loss adapts the prompts to the downstream objective, while the routing regularizer keeps the adaptation aware of source-target divergence, reducing the risk of relying on mismatched source-domain knowledge.


\section{Algorithms and Complexity Analysis}
\label{app:alg}

We summarize the pre-training procedure of \modelname~in Algorithm~\ref{alg:pretrain} and the fine-tuning procedure in Algorithm~\ref{alg:finetune}. The analysis below shows that both stages scale linearly with graph size, while the additional costs introduced by semantic-temporal decoupling, divergence-aware routing, and prompt learning remain lightweight.

\input{table/alg_pretrain}

\subsection{Complexity Analysis of Algorithm~\ref{alg:pretrain}}

The pre-training stage consists of the semantic branch and the temporal branch. Their costs are mainly determined by feature alignment, message passing, and temporal aggregation.

\begin{itemize}[leftmargin=*]
    \item \textbf{Semantic Branch} (lines 3-9)\textbf{:} For each source domain, feature alignment projects the input feature matrix into the shared semantic space, with complexity $\mathcal{O}\left(|\mathcal{V}_i|d_0d\right)$. The GNN-based aggregation, InfoNCE-based IB objective, and link-based pre-training loss are computed over graph edges, leading to complexity $\mathcal{O}\left(|\mathcal{E}_i|d\right)$. Therefore, the overall cost of the semantic branch for domain $i$ is:
    \begin{equation}
        \mathcal{O}\left(|\mathcal{E}_i|d + |\mathcal{V}_i|d_0d\right).
    \end{equation}

    \item \textbf{Temporal Branch} (lines 11-20)\textbf{:} The temporal branch performs relative-time encoding, TGAT-based temporal propagation, and lightweight adapter refinement. Relative-time encoding and temporal message passing scale with the number of temporal edges, yielding $\mathcal{O}\left(|\mathcal{E}_i|d\right)$. The timers and residual adapters introduce additional linear costs, denoted as $\mathcal{O}\left(|\mathcal{V}_i|r\right)$, where $r \ll d$ is the adapter bottleneck dimension. Thus, the temporal branch also remains linear with respect to the node and edge counts.
\end{itemize}

Aggregating over $n$ source domains, let $N=\sum_i|\mathcal{V}_i|$ and $M=\sum_i|\mathcal{E}_i|$ denote the total number of nodes and edges, respectively. The overall pre-training complexity is:
\begin{equation}
    \mathcal{O}\left(Md + Nd_0d\right).
\end{equation}
The dominant computational costs arise from edge-level message passing and node-level feature alignment. Other components, including IB regularization, timers, and adapters, introduce only minor linear overhead. Therefore, \modelname~preserves scalability for large multi-domain dynamic graphs.

\input{table/alg_funetune}

\subsection{Complexity Analysis of Algorithm~\ref{alg:finetune}}

The fine-tuning stage adapts the frozen pre-trained encoders to a target dynamic graph through divergence-aware routing and lightweight prompt learning. Its computation mainly arises from prototype construction, prompt generation, and task-specific optimization.

\begin{itemize}[leftmargin=*]
    \item \textbf{Prototype Construction and Routing} (lines 3-14)\textbf{:} Computing source prototypes and target semantic-temporal descriptors requires aggregating semantic and temporal representations. The cost is linear in the number of source domains and the target graph size, namely $\mathcal{O}\left(nd + M^{\mathcal T}d + N^{\mathcal T}d\right)$, where $N^{\mathcal T}$ and $M^{\mathcal T}$ denote the number of target nodes and temporal edges.

    \item \textbf{Prompt Generation and Injection} (lines 15-19)\textbf{:} The semantic and temporal prompts are generated by two compact MLPs, whose complexity is $\mathcal{O}\left(dp\right)$ with prompt dimension $p \ll d$. Prompt injection only augments the encoder inputs and does not update the frozen backbone, so it introduces negligible additional cost.

    \item \textbf{Task Optimization} (lines 20-24)\textbf{:} The main cost comes from forward propagation through the frozen semantic and temporal encoders, with complexity $\mathcal{O}\left(M^{\mathcal T}d + N^{\mathcal T}d\right)$. Computing the node classification or link prediction loss and updating prompt parameters add only linear overhead.
\end{itemize}

Combining above, the overall fine-tuning complexity is:
\begin{equation}
    \mathcal{O}\left(nd + M^\mathcal{T} d + N^\mathcal{T} d + dp\right).
\end{equation}
Since $n$ and $p$ are typically small, fine-tuning scales linearly with the target graph size. Moreover, because the pre-trained encoders are frozen and only lightweight prompts and routing-related parameters are optimized, \modelname~supports efficient and parameter-efficient cross-domain adaptation.

%% file: fig/framework_marked.tex
\begin{tikzpicture}[very thick, black]
\node[draw=none,fill=none] (image) at (0,0){\includegraphics{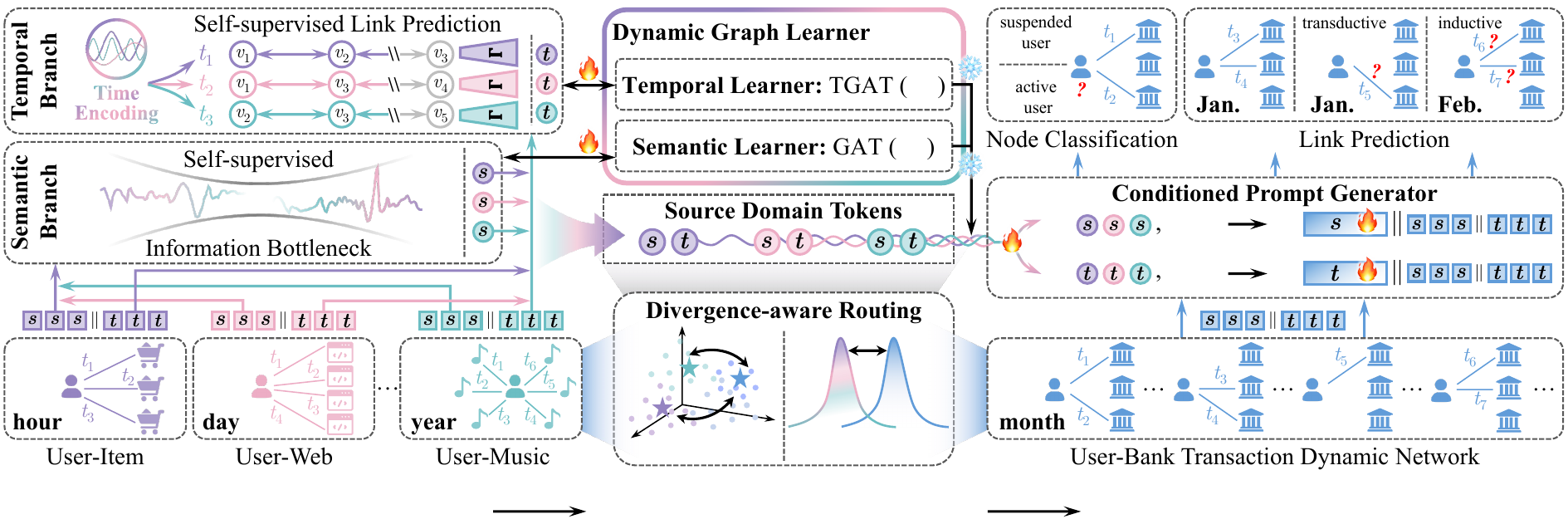}};


\node[align=center, color=black, font=\LARGE] at (-15.3, -1.99) {$\mathcal{G}_1^\mathcal{S}$};
\node[align=center, color=black, font=\LARGE] at (-11.45, -1.99) {$\mathcal{G}_2^\mathcal{S}$};
\node[align=center, color=black, font=\LARGE] at (-7.2, -1.99) {$\mathcal{G}_n^\mathcal{S}$};
\node[align=center, color=black, font=\LARGE] at (4.84, -1.99) {$\mathcal{G}^\mathcal{T}$};

\node[align=center, color=black, font=\Large] at (-1.75, -3.76) {$\operatorname{Div}_\text{s}$};
\node[align=center, color=black, font=\Large] at (1.7, -3.76) {$\operatorname{Div}_\text{t}$};

\node[align=center, color=black, font=\Large] at (-15.4, -0.48) {$\widehat{\mathbf{X}}^\mathcal{S}$};
\node[align=center, color=black, font=\Large] at (-14.3, 1.34) {$\widehat{\mathbf{X}}^\mathcal{S}$};
\node[align=center, color=black, font=\Large] at (-10.64, 1.33) {$\mathbf{Z}^\mathcal{S}$};
\node[align=center, color=black, font=\Large] at (-7, 1.33) {$\mathbf{Z}^{\mathcal{S}+}$};
\node[align=center, color=black, font=\Large] at (-7.35, 2.16) {$\mathbf{Z}^{\mathcal{S}+}$};
\node[align=center, color=black, font=\Large] at (-7.35, 0.5) {$\mathbf{Z}^{\mathcal{S}+}$};
\node[align=center, color=black, font=\Large] at (-4.4, -0.48) {$\mathbf{Z}^{\mathcal{S}}\!(t)$};

\node[align=center, color=black, font=\Large] at (2.8, 3.62) {$f_\text{t}$};
\node[align=center, color=black, font=\Large] at (2.6, 2.33) {$f_\text{s}$};
\node[align=center, color=black, font=\Large] at (2.7, 4.688) {$h\!=\!g\!\circ\! f$};

\node[align=center, color=black, font=\Large] at (5.63, 0.8) {$\boldsymbol{\mathcal{P}}\!_\text{s}\big($};
\node[align=center, color=black, font=\Large] at (8.44, 0.8) {$\operatorname{Div}_\text{s}\!\big)$};
\node[align=center, color=black, font=\Large] at (10.2, 0.83) {$\mathbf{p}_\text{s}^\mathcal{T}$};

\node[align=center, color=black, font=\Large] at (5.621, -0.2) {$\boldsymbol{\mathcal{P}}\!_\text{t}\big($};
\node[align=center, color=black, font=\Large] at (8.425, -0.182) {$\operatorname{Div}_\text{t}\!\big)$};
\node[align=center, color=black, font=\Large] at (10.2, -0.16) {$\mathbf{p}_\text{t}^\mathcal{T}$};

\node[align=center, color=black, font=\LARGE, scale=0.97] at (-10, -5.05) {\hyperref[sec:pretrain]{\textbf{(1)}} \textbf{Dual-Branch Pre-training}};
\node[align=center, color=black, font=\LARGE, scale=0.97] at (0, -5.05) {\hyperref[sec:moe]{\textbf{(2)}} \textbf{Cross-Domain Adaptation}};
\node[align=center, color=black, font=\LARGE, scale=0.97] at (9.95, -5.05) {\hyperref[sec:finetune]{\textbf{(3)}} \textbf{Conditioned Fine-tuning}};

\end{tikzpicture}

%% file: table/alg_pretrain.tex
\begin{algorithm}
\caption{Overall pre-training pipeline of \modelname.}
\label{alg:pretrain}
\KwIn{Source dynamic graphs $\{\mathcal{G}_i^{\mathcal{S}}\}_{i=1}^n$ from domain $\{D^{\mathcal{S}}\}$ with labels $\{Y^{\mathcal{S}}\}$; Each $\mathcal{G}_i^{\mathcal{S}}$ is with node set $\mathcal{V}_i$, temporal edge set $\mathcal{E}_i=\{(u,v,t)\}$, and a feature matrix $\mathbf{X}_i\in\mathbb{R}^{|\mathcal{V}_i|\times d_i}$; Hidden dimension $d$; Learning rate $\eta$; Temperature parameter $\tau$; Hyper-parameter $\beta, \lambda_1$; Pre-training epochs $E_1, E_2, E_3$.}
\KwOut{Dynamic graph encoder $f=f_\text{s} \circ f_\text{t}$ with parameters $\boldsymbol{\theta}^\star=\boldsymbol{\theta}^\star_\text{s}\circ \boldsymbol{\theta}^\star_\text{t}$; Domain-specific timers $\{\boldsymbol{\mathcal{T}}_i\}_{i=1}^n$ with parameters $\{\boldsymbol{\Phi}_i\}_{i=1}^n$; Domain-specific adapters $\{\boldsymbol{\Gamma}_i\}_{i=1}^n$ with parameters $\{\boldsymbol{\Psi}_i\}_{i=1}^n$; Semantic tokens $\{\mathbf{Z}_i^{\mathcal S}\}_{i=1}^n$; Temporal tokens $\{\widetilde{\mathbf{Z}}_i^{\mathcal S}(t)\}_{i=1}^n$.}
\BlankLine
Initialize all learnable parameters randomly\;
\textcolor{gray}{\tcp{Semantic Branch}}
\For{$e_1=1,2,\cdots, E_1$}{
Remove timestamps and merge edges for each $\mathcal{G}_i^\mathcal{S}$\;
Calculate semantic tokens $\mathbf{Z}_i^\mathcal{S} \gets$ Eq.~\eqref{eq:ssib_next} with $f_\text{s}$\;
Calculate IB objective $\mathcal{L}_\text{SS-IB} \gets$ Eq.~\eqref{eq:ssib}\;
Calculate task loss $\mathcal{L}_\text{pre}\big(\boldsymbol{\theta}_\text{s},\mathbf{Z}^{\mathcal S}\big) \gets$ Eq.~\eqref{eq:pre_sem}\;
Calculate semantic loss $\mathcal{L}_\text{pre-sem} \gets$ Eq.~\eqref{eq:pre_sem_overall}\;
Update $\boldsymbol{\theta}_\text{s}$ by minimizing $\mathcal{L}_\text{pre-sem}$ and back-propagation with learning rate $\eta$\;
}
\BlankLine
\textcolor{gray}{\tcp{Temporal Branch}}
\For{$e_2=1,2,\cdots, E_2$}{
For each edge $(u,v,t)\in\mathcal E_i$ with relative time $t$, calculate its relative-time encoding $\mathbf{r}^{\mathcal S}_i(t)\!\gets\!$ Eq.~\eqref{eq:rte} with $\{\boldsymbol{\mathcal{T}}_i\}_{i=1}^n$\;
\For{$e_3=1,2,\cdots, E_3$}{
Update the temporal tokens $\widetilde{\mathbf{Z}}_i^\mathcal{S}(t) \gets$ Eq.~\eqref{eq:update} with $f_\text{t}$ and $\{\boldsymbol{\Gamma}_i\}_{i=1}^n$ for each $\mathcal{G}_i^\mathcal{S}$\;
\BlankLine
\textcolor{gray}{\tcp{Phase I}}
Sample mini-batch $\mathcal B_i \subset \mathcal G_i^{\mathcal S}$ for each $G_i^{\mathcal S}$\;
Calculate the temporal pre-training loss for TGAT backbone $\mathcal{L}_\text{pre-tem}^{(i)}\gets$ Eq.~\eqref{eq:pre_loss_temp}\;
Update $\boldsymbol{\theta}_\text{t}, \boldsymbol{\Phi}_i$ by minimizing $\mathcal{L}_\text{pre-tem}^{(i)}$ and back-propagation with learning rate $\eta$\;
\BlankLine
\textcolor{gray}{\tcp{Phase II}}
Update $\boldsymbol{\Psi}_i$ by minimizing $\mathcal{L}_\text{pre-tem}^{(i)}$ and back-propagation with learning rate $\eta$, while keeping the updated $\nabla_{\boldsymbol{\theta}_\text{t}}=\boldsymbol{0}$ and $\nabla_{\boldsymbol{\Phi}_i}=\boldsymbol{0}$\;
}
}

\end{algorithm}

%% file: table/alg_funetune.tex
\begin{algorithm}
\caption{Overall fine-tuning pipeline of \modelname.}
\label{alg:finetune}
\KwIn{Target dynamic graph $\mathcal G^{\mathcal T}$ with a feature matrix $\mathbf{X}\in\mathbb{R}^{|\mathcal{V}_i|\times d_i}$; Dynamic graph encoder $f=f_\text{s} \circ f_\text{t}$ with parameters $\boldsymbol{\theta}^\star=\boldsymbol{\theta}^\star_\text{s}\circ \boldsymbol{\theta}^\star_\text{t}$; Timers $\{\boldsymbol{\mathcal{T}}_i\}_{i=1}^n$ with parameters $\{\boldsymbol{\Phi}_i\}_{i=1}^n$; Adapters $\{\boldsymbol{\Gamma}_i\}_{i=1}^n$ with parameters $\{\boldsymbol{\Psi}_i\}_{i=1}^n$; Semantic tokens $\{\mathbf{Z}_i^{\mathcal S}\}_{i=1}^n$; Temporal tokens $\{\widetilde{\mathbf{Z}}_i^{\mathcal S}(t)\}_{i=1}^n$; Hidden dimension $d$; Learning rate $\eta$; Temperature $\tau$; Hyper-parameter $\gamma,\lambda_\text{s},\lambda_\text{t},\lambda_r,\lambda_2,\lambda_3$; Fine-tuning epochs $E_4$.}
\KwOut{Fine-tuned dynamic graph learner $h^\star=g^\star\circ f^\star$, where $f^\star=f_\text{s}^\star \circ f_\text{t}^\star$ with parameters $\{\boldsymbol{\theta}^\star_\text{s}, \boldsymbol{\theta}^\star_\text{t},$ $\boldsymbol{\Omega}^\star_\text{s},\boldsymbol{\Omega}^\star_\text{t}\}$.}
\BlankLine
Initialize all learnable parameters randomly\;
\For{$e_4=1,2,\cdots, E_4$}{
\textcolor{gray}{\tcp{Expert Pool and Prototypes}}
Calculate semantic prototype $\mathbf{s}_i^{\mathcal S} \gets$ Eq.~\eqref{eq:prototype}\;
Calculate temporal prototype $\mathbf{t}_i^{\mathcal S} \gets$ Eq.~\eqref{eq:prototype}\;
Calculate target semantic descriptor $\mathbf{s}^{\mathcal T} \gets$ Eq.~\eqref{eq:s_descrip}\;
Calculate target temporal descriptor $\mathbf{t}^{\mathcal T} \gets$ Eq.~\eqref{eq:t_descrip}\;
\textcolor{gray}{\tcp{Divergence-aware Routing}}
Calculate semantic divergence $\operatorname{Div}_{\text{s}}\big(D^{\mathcal T},D_i^{\mathcal S}\big) \gets$ Eq.~\eqref{eq:s_div} between target and source domain\;
Calculate temporal divergence $\operatorname{Div}_{\text{t}}\big(D^{\mathcal T},D_i^{\mathcal S}\big) \gets$ Eq.~\eqref{eq:t_div} between target and source domain\;
Calculate combined divergence $\operatorname{Div}_i\gets$ Eq.~\eqref{eq:div_comb}\;
\textcolor{gray}{\tcp{Routing Optimization}}
Initialize routing weights $\boldsymbol{\alpha}\gets$ Eq.~\eqref{eq:routing}\;
Calculate routing regularizer $\mathcal R_{\text{routing}}\gets$ Eq.~\eqref{eq:routing_r}\;
\textcolor{gray}{\tcp{Conditioned Prompt Generation}}
Initialize semantic and temporal prompts $(\mathbf{p}_\text{s}^\mathcal{T},\mathbf{p}_\text{t}^\mathcal{T}) \gets$ Eq.~\eqref{eq:prompt_1}, Eq.~\eqref{eq:prompt_2} with prompt generators $(\boldsymbol{\mathcal{P}}_\text{s}, \boldsymbol{\mathcal{P}}_\text{t})$\;
\textcolor{gray}{\tcp{Prompt Injection and Fusion}}
Inject prompts at encoders, and obtain semantic and temporal node embeddings $\mathbf{Z}^{\mathcal T}_{\mathbf{p}}, \widetilde{\mathbf{Z}}^{\mathcal T}_{\mathbf{p}}(t) \gets$ Eq.~\eqref{eq:inject}\;
Obtain prompted node embeddings $\mathbf{H}^{\mathcal T}(t) \gets$ Eq.~\eqref{eq:fuse} by a lightweight fusion\;
\textcolor{gray}{\tcp{Task Objectives}}
Calculate node classification loss $\mathcal L_{\text{node}}\gets$ Eq.~\eqref{eq:node}\;
Calculate link prediction loss $\mathcal L_{\text{link}}\gets$ Eq.~\eqref{eq:link}\;
Calculate overall fine-tuning loss $\mathcal L_{\text{ftn}} \gets$ Eq.~\eqref{eq:ftn_overall}\;
Update $\boldsymbol{\Omega}_\text{s},\boldsymbol{\Omega}_\text{t}$ by minimizing $\mathcal L_{\text{ftn}}$ and back-propagation with learning rate $\eta$\;
}

\end{algorithm}

%% file: 5_experiment.tex
\section{Experiment}
\label{sec:exp}
In this section, we conduct extensive experiments to evaluate the proposed \modelname\footnote{\url{https://github.com/RingBDStack/DyGFM}.} over the following research questions:
\begin{itemize}[leftmargin=1.5em]
    \item \textbf{\textit{RQ1:}} How effective does the proposed \modelname~transfer across different datasets and domains? ($\rhd$~Section~\ref{sec:rq1})
    \item \textbf{\textit{RQ2:}} Which module contributes most? ($\rhd$~Section~\ref{sec:rq2})
    \item \textbf{\textit{RQ3:}} How time-efficient in fine-tuning? ($\rhd$~Section~\ref{sec:rq3})
    \item \textbf{\textit{RQ4:}} How interpretable are the routing weights among the source domain experts? ($\rhd$~Section~\ref{sec:rq4})
    \item \textbf{\textit{RQ5:}} How do node embeddings evolve? ($\rhd$~Section~\ref{sec:rq5})
    \item \textbf{\textit{RQ6:}} How sensitive to hyper-parameters? ($\rhd$~Section~\ref{sec:rq6})
\end{itemize}

\subsection{Experimental Settings}
\label{sec:exp_setting}

\input{table/dataset}
\subsubsection{\textbf{Datasets}}

We conduct experiments on four widely used real-world continuous-time dynamic graph datasets, covering different domains with heterogeneous semantics and temporal granularities. The statistics are summarized in Table~\ref{tab:dataset}.
\begin{itemize}[leftmargin=1.5em]
    \item \texttt{Wikipedia}~\cite{kumar2019predicting}: a user-page editing network, where temporal edges denote editing activities and dynamic labels indicate whether a user is temporarily banned from editing.
    \item \texttt{Reddit}~\cite{kumar2019predicting}: a user-post interaction network on subreddits, where edges represent users commenting on posts and labels denote whether a user is banned from posting.
    \item \texttt{MOOC}~\cite{kumar2019predicting}: a student-course activity network from an online education platform, where temporal edges record learning behaviors and labels indicate whether a student drops out after an activity.
    \item \texttt{Genre}~\cite{huang2023temporal}: a user-genre listening network, where edges represent users listening to music genres over time and labels indicate users' preferred music genres.
\end{itemize}

\subsubsection{\textbf{Baselines}}

We compare \modelname~with 12 state-of-the-art baselines from four categories, covering conventional DGNNs, dynamic graph pre-training methods, static graph foundation models, and dynamic graph prompting methods. For fair comparison, all baselines are adapted to the same pre-training and fine-tuning settings whenever applicable.

\begin{itemize}[leftmargin=1.5em]
    \item \textbf{Naive DGNNs:} \texttt{TGN}~\cite{rossi2020temporal} uses memory modules to capture historical interactions, \texttt{TGAT}~\cite{xu2020inductive} employs time encoding and temporal attention for inductive representation learning, \texttt{ROLAND}~\cite{you2022roland} adapts static GNNs to dynamic settings through recurrent state updates, and \texttt{TREND}~\cite{wen2022trend} models temporal event and node dynamics with Hawkes-process-based mechanisms. Since these methods are not originally designed for multi-domain pre-training, we train them on the available source-domain dynamic graphs under the same protocol as \modelname, so that they can access the same source-domain data.

    \item \textbf{Dynamic Graph Pre-training:} \texttt{DDGCL}~\cite{tian2021self} learns temporal consistency through dynamic contrastive learning with debiased sampling, while \texttt{CPDG}~\cite{bei2024cpdg} performs contrastive pre-training by jointly modeling structural and temporal signals. These methods are originally developed for dynamic graph pre-training, but mainly in single-domain settings. For fair comparison, we extend them to the multi-domain setting by mixing the available source-domain events during pre-training, and then adapt them to downstream target tasks with the same few-shot splits as \modelname.

    \item \textbf{Static GFMs:} \texttt{GraphPrompt}~\cite{liu2023graphprompt} introduces a unified pre-train-prompt paradigm for static graphs, \texttt{ProG}~\cite{zi2024prog} provides graph prompting protocols, \texttt{GCOPE}~\cite{zhao2024all} improves cross-domain adaptation with coordinator-based transfer, and \texttt{SAMGPT}~\cite{yu2025samgpt} aligns multi-domain graphs with structure tokens and dual prompts. Since these methods are designed for static graphs, we first remove timestamps from each available source dynamic graph and merge temporal edges into a static graph for multi-domain pre-training. During downstream evaluation, we attach a TGAT~\cite{xu2020inductive} temporal encoder to process temporal interactions, so that static GFMs can be evaluated on the same dynamic node classification and link prediction tasks. This adaptation provides static GFMs with access to multi-domain graph knowledge while allowing them to handle temporal downstream inputs as fairly as possible.

    \item \textbf{Dynamic Graph Prompting Methods:} \texttt{TIGPrompt}~\cite{chen2024prompt} introduces a temporal prompt generator to bridge temporal and semantic gaps in temporal interaction graphs, while \texttt{DyGPrompt}~\cite{yu2025nodetime} designs node-time conditional prompts to capture evolving node-time patterns. These methods are closer to dynamic graph foundation modeling than conventional DGNNs, but they are mainly developed under single-domain pre-training and adaptation settings. For fair comparison, we extend them to the multi-domain scenario by mixing the available source domains during pre-training. Their prompt tuning procedures follow the original designs, with the same support, validation, and test splits as \modelname.
\end{itemize}

\subsubsection{\textbf{Pre-training and Fine-tuning Settings}}

We evaluate \modelname~and baselines under two adaptation settings:
\begin{itemize}[leftmargin=1.5em]
    \item \textbf{ASDA} (\underline{\textbf{A}}ll-\underline{\textbf{S}}een-\underline{\textbf{D}}omain-\underline{\textbf{A}}daptation): All four domains are included during pre-training, and each domain is later used for downstream adaptation.
    \item \textbf{LODO} (\underline{\textbf{L}}eave-\underline{\textbf{O}}ne-\underline{\textbf{D}}omain-\underline{\textbf{O}}ut): One domain is excluded from pre-training and used only for downstream adaptation, which evaluates the transferability to unseen domains.
\end{itemize}

For each domain, the events (temporal edges) are ordered chronologically to preserve causality. We use the first 80\% of events for pre-training and the remaining 20\% for downstream evaluation. The downstream portion is further split into a 1\% support set, a 1\% validation set, and an 18\% test set. This setting simulates few-shot adaptation, where only limited target-domain labels are available.

During pre-training, \modelname~utilizes a two-layer GAT as the semantic encoder and a two-layer TGAT as the temporal encoder. The model is pre-trained for up to 5,000 epochs with an early stopping strategy. During fine-tuning, all pre-trained encoders are frozen, and only the divergence-aware router and divergence-conditioned prompt generators are updated. Fine-tuning is conducted for up to 300 epochs with early stopping based on validation AUC. We use Adam for optimization, with the learning rate in $[$10$^{-\text{4}}$,10$^{-\text{1}}]$ and weight decay in $[$10$^{-\text{5}}$,10$^{-\text{1}}]$. The fine-tuning hyper-parameters $\lambda_{\text{s}}$, $\lambda_{\text{t}}$, and $\lambda_r$ are in $[\text{0}$, $\text{1}]$, $[\text{0}$, $\text{1}]$, and $[\text{0.001}$, $\text{1}]$, respectively.

\subsubsection{\textbf{Downstream Tasks and Evaluation}}

We evaluate two downstream tasks: temporal node classification and temporal link prediction. For temporal node classification, dynamic labels are aligned with event timestamps, and the model predicts the label of a node at the corresponding time. For temporal link prediction, positive and negative samples are constructed at each timestamp, and performance is evaluated under both transductive and inductive settings. We report AUC-ROC (\%) with the mean and standard deviation over five repeated runs.

\subsubsection{\textbf{Implementation Environment}}

All experiments are conducted on a server with Ubuntu 20.04 LTS, an Intel(R) Xeon(R) Platinum 8358 CPU@2.60GHz, 1TB DDR4 memory, and an NVIDIA Tesla A100 GPU with 80GB memory. The software environment includes CUDA 10.1, Python 3.8.12, PyTorch 1.9.1, and PyTorch Geometric 2.0.1.

\input{table/res_main}

\subsection{RQ1: Transfer across Domains and Tasks}
\label{sec:rq1}

\textit{RQ1} evaluates the effectiveness of \modelname~in transferring across different domains and downstream tasks under different adaptation settings. We conduct experiments on both link prediction and node classification over {four} continuous-time dynamic graph datasets, considering {three} experimental setups.

\subsubsection{\textbf{ASDA Setting} (Table~\ref{tab:main})}
In this setting, all domains are included during pre-training, and fine-tuning is performed on the downstream portion of the same domains. This setting evaluates in-domain transfer efficiency under the few-shot scenario, where only 1\% of the target labels are available.
The results show that \modelname~consistently achieves the best performance across all datasets and tasks. The largest improvement appears on \texttt{MOOC} node classification, where \modelname~outperforms \texttt{DyGPrompt} by 5.0\%. Compared with single-domain dynamic prompt baselines such as \texttt{DyGPrompt} and \texttt{TIGPrompt}, \modelname~shows stronger in-domain generalization by leveraging divergence-aware routing, which prevents overfitting to a single source. The consistent gains under both transductive and inductive link prediction further indicate that semantic-temporal decoupling brings complementary benefits: semantic prototypes provide global structural priors, while temporal tokens capture dynamic evolution.

\subsubsection{\textbf{LODO Setting} (Table~\ref{tab:main})}
In this setting, one domain is excluded from pre-training and used only for fine-tuning, directly testing cross-domain generalization to unseen domains.
The results show that \modelname~again achieves the highest scores on most datasets and tasks, with notable relative improvements of 8.0\% on \texttt{MOOC} node classification and 7.1\% on \texttt{Reddit} transductive link prediction. These improvements demonstrate that divergence-aware routing effectively mitigates negative transfer by dynamically adjusting domain mixture weights according to semantic and temporal discrepancies. Compared with static GFMs such as \texttt{GCOPE} and \texttt{SAMGPT}, \modelname~exhibits stronger adaptability to unseen dynamic patterns and irregular temporal shifts, since divergence-conditioned prompts are expected to selectively amplify transferable knowledge from semantically similar source domains. Overall, the results confirm that \modelname~not only performs well in in-domain few-shot scenarios, but also generalizes effectively across domains with distinct semantics and temporal granularities.

\textbf{Key Takeaway:}
\modelname~consistently excels under both ASDA and LODO settings, demonstrating strong few-shot adaptability and robust cross-domain generalization.

\input{table/res_klodo}


\subsubsection{$K$\textbf{-LODO Setting} (Table~\ref{tab:klodo})}
We further consider a more challenging Leave-One-or-$K$-Domains-Out setting ($K>1$), where some domains are excluded from pre-training. This setup evaluates the robustness of \modelname~under larger domain gaps.

The results show that \modelname~maintains satisfying performance even when multiple domains are removed from pre-training, with only minor degradation compared with the single-domain LODO case. The averaged AUC-ROC across all target domains remains competitive for node classification, indicating stable cross-domain generalization. We also observe that different source-domain combinations exhibit complementary effects. Semantically or temporally related domains, such as \texttt{Wikipedia} + \texttt{Reddit}, yield higher transferability, whereas more dissimilar pairs, such as \texttt{MOOC} + \texttt{Genre}, show reduced gains. This suggests that semantic proximity helps bridge structural and temporal gaps. In addition, increasing $K$ does not always improve performance. Although multi-source pre-training enriches domain coverage, it may also introduce conflicting temporal or semantic dynamics. The slight variance across target domains, such as stronger results on \texttt{MOOC} and weaker results on \texttt{Reddit}, suggests that adaptive routing and divergence-conditioned prompting can balance universality and domain specificity without overfitting.

\subsection{\textit{RQ2:} Ablation Study}
\label{sec:rq2}

\begin{figure}[t]
    \centering
    \includegraphics[width=0.325\textwidth]{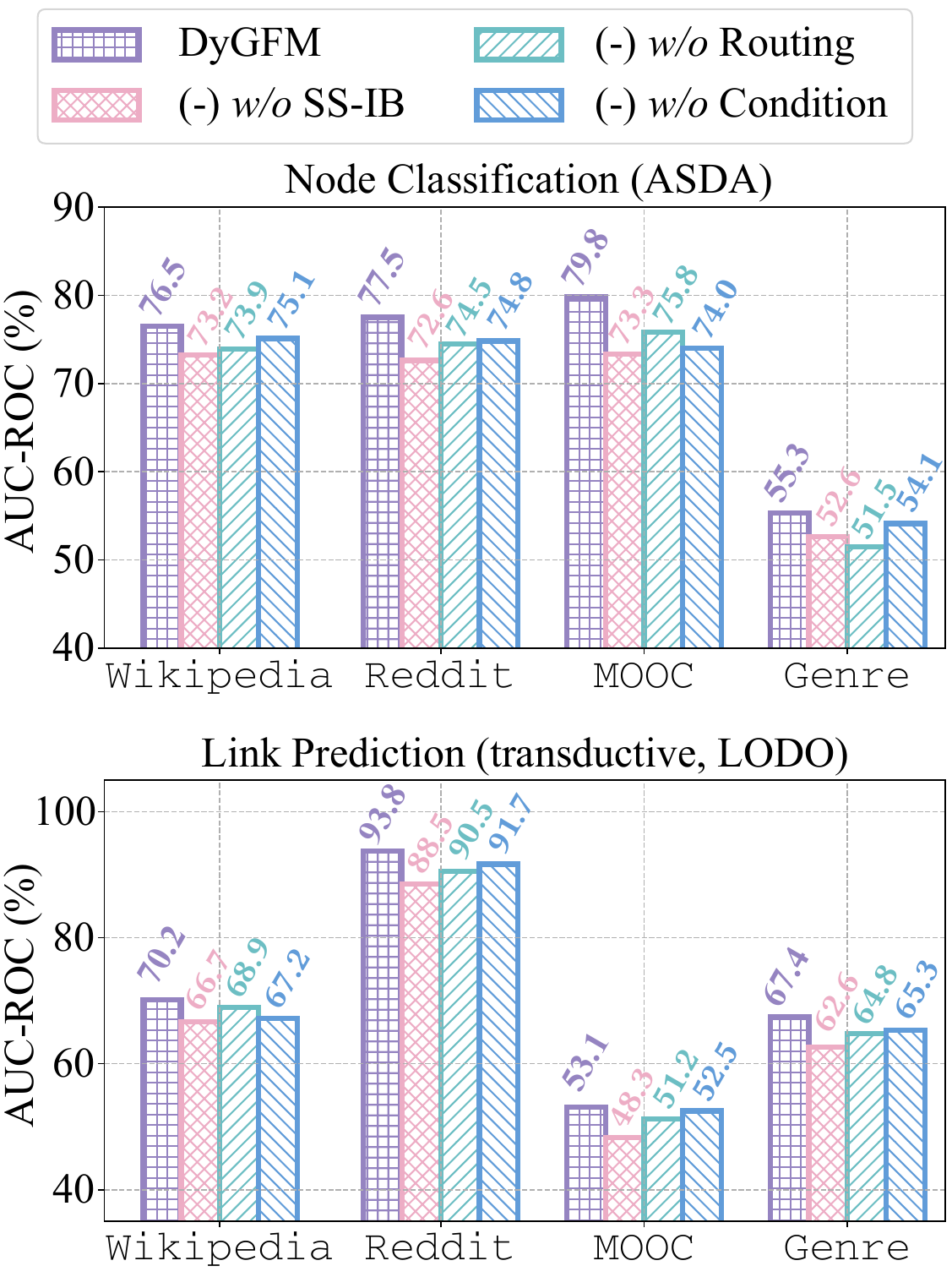}
    \caption{Ablation study on node classification and transductive link prediction.}
    \label{fig:abla}
\end{figure}

We conduct ablation studies on three core modules:
\begin{itemize}[leftmargin=1.5em]
    \item \textbf{\modelname~(\textit{w/o} SS-IB):} We disable the information bottleneck aligner in the semantic branch (Eq.~\eqref{eq:ssib}), preventing filtering redundant domain-specific noise during pre-training.
    \item \textbf{\modelname~(\textit{w/o} Routing):} We remove the divergence-aware router and directly average source experts (Eq.~\eqref{eq:routing}), ablating dynamic balancing between universality and specificity.
    \item \textbf{\modelname~(\textit{w/o} Condition):} We replace the divergence-conditioned prompt generator with a randomly initialized prompt (Eq.~\eqref{eq:prompt_1}, Eq.~\eqref{eq:prompt_2}), removing the conditioning on both specific task and divergence.
\end{itemize}
Results in Figure~\ref{fig:abla} show that removing ``SS-IB'' noticeably degrades performance across all datasets. This confirms that the self-supervised information bottleneck in the semantic branch effectively suppresses noise while preserving transferable semantic priors for downstream tasks.
Removing ``Routing'' leads to moderate but consistent performance drops, \eg, –3.4\% on average for node classification and –2.3\% for link prediction. Without divergence-aware routing, the model cannot dynamically select suitable source-domain experts under domain shifts, which weakens its cross-domain adaptability.
Removing ``Condition'' also reduces performance, particularly on domains with stronger heterogeneity such as \texttt{Genre}. This indicates that divergence-conditioned prompts are important for adjusting domain-specific biases and enabling fine-grained adaptation under different semantic and temporal dynamics.

\textbf{Key Takeaway:}
The three modules provide complementary benefits: semantic compression stabilizes transferable knowledge, routing improves source-domain relevance, and conditioned prompting refines adaptation granularity.

\begin{figure*}[!t]
    \begin{minipage}{0.325\textwidth}
        \centering
        \includegraphics[width=\textwidth]{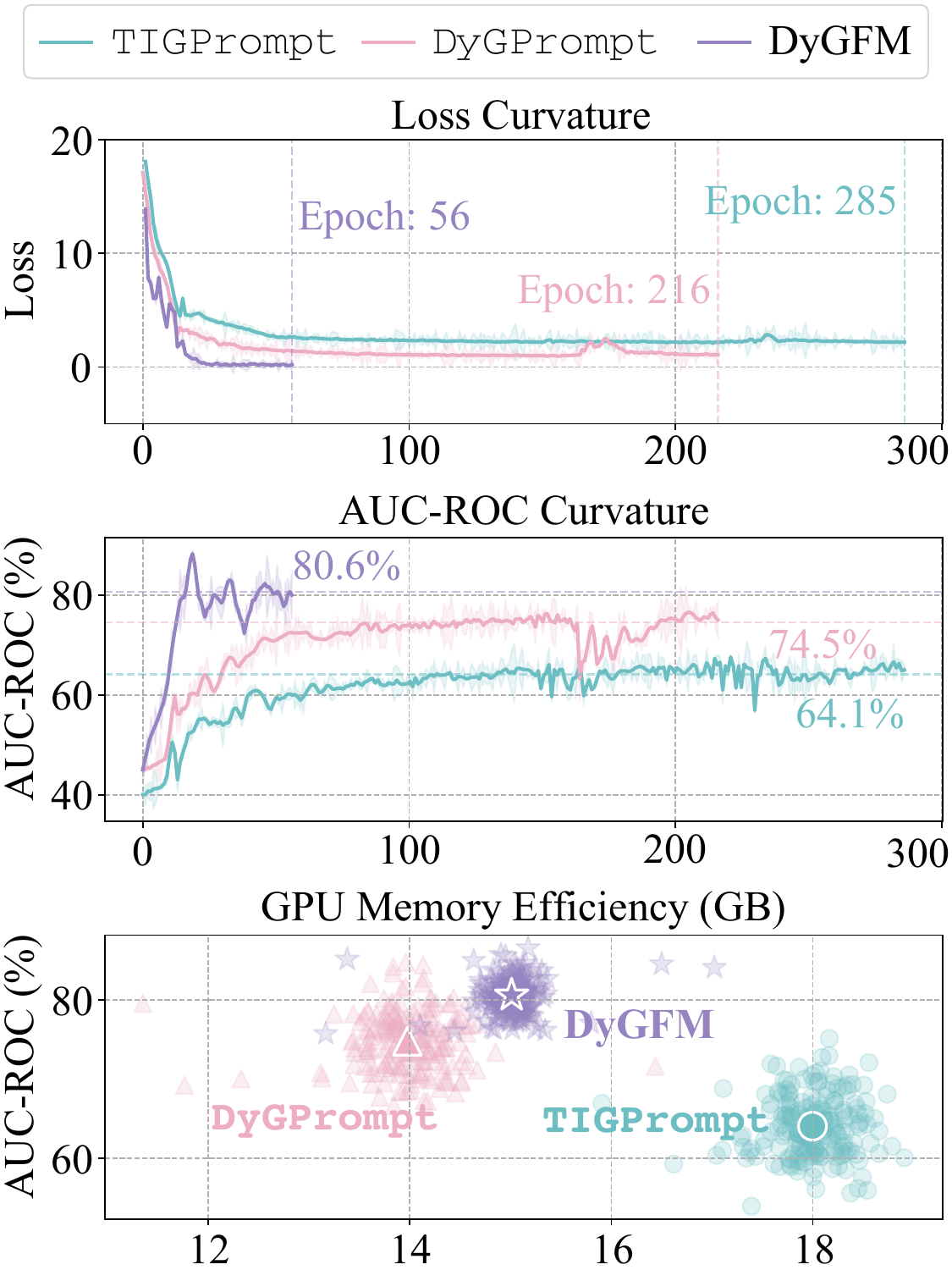}
        \caption{Efficiency Analysis on \texttt{Reddit}.}
        \label{fig:curve}
    \end{minipage} 
    \hfill
    \begin{minipage}{0.325\textwidth}
        \centering
        \includegraphics[width=\textwidth]{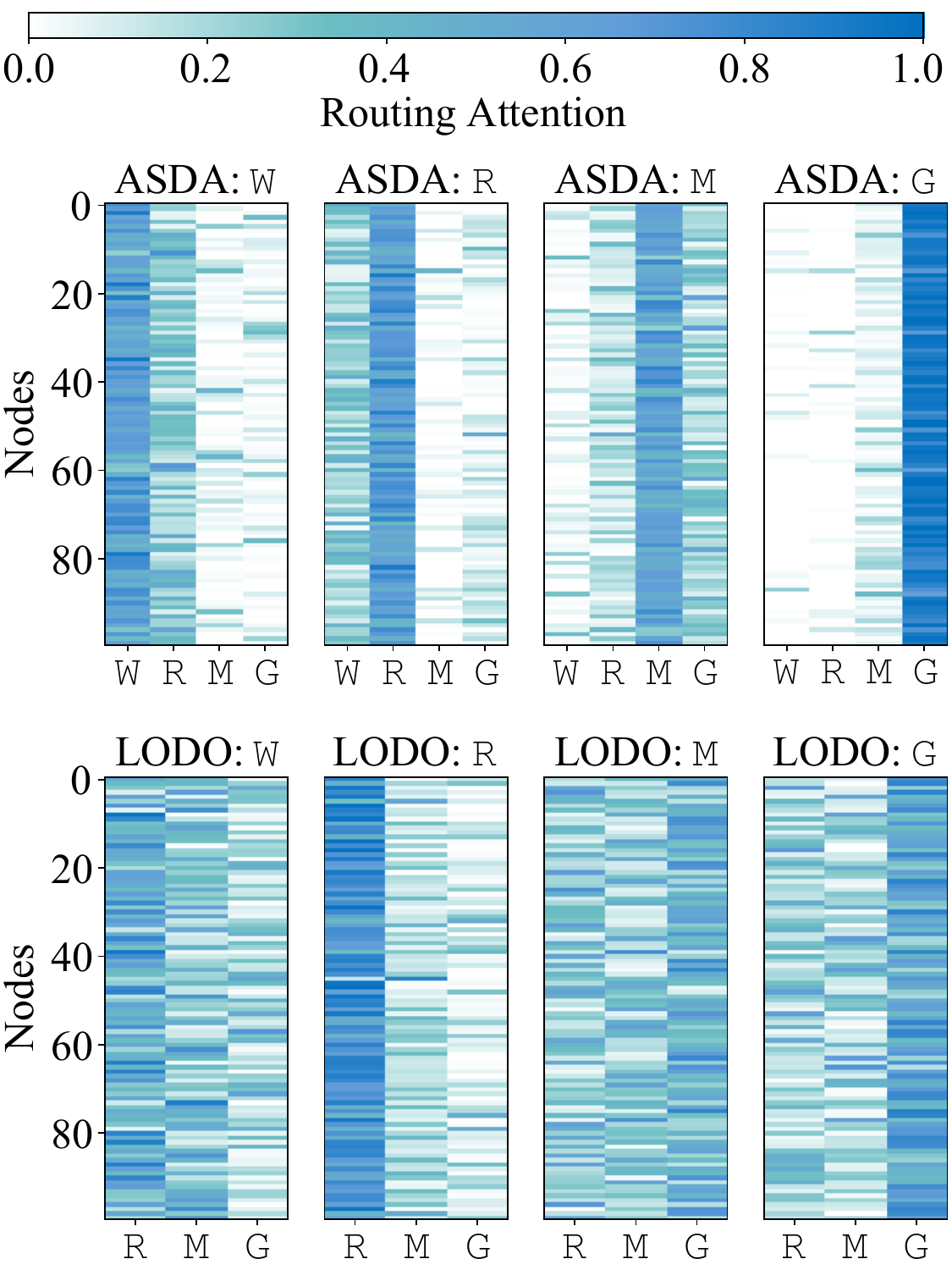}
        \caption{Routing Visualizations.}
        \label{fig:heat}
    \end{minipage} 
    \hfill
    \begin{minipage}{0.325\textwidth}
        \centering
        \includegraphics[width=\textwidth]{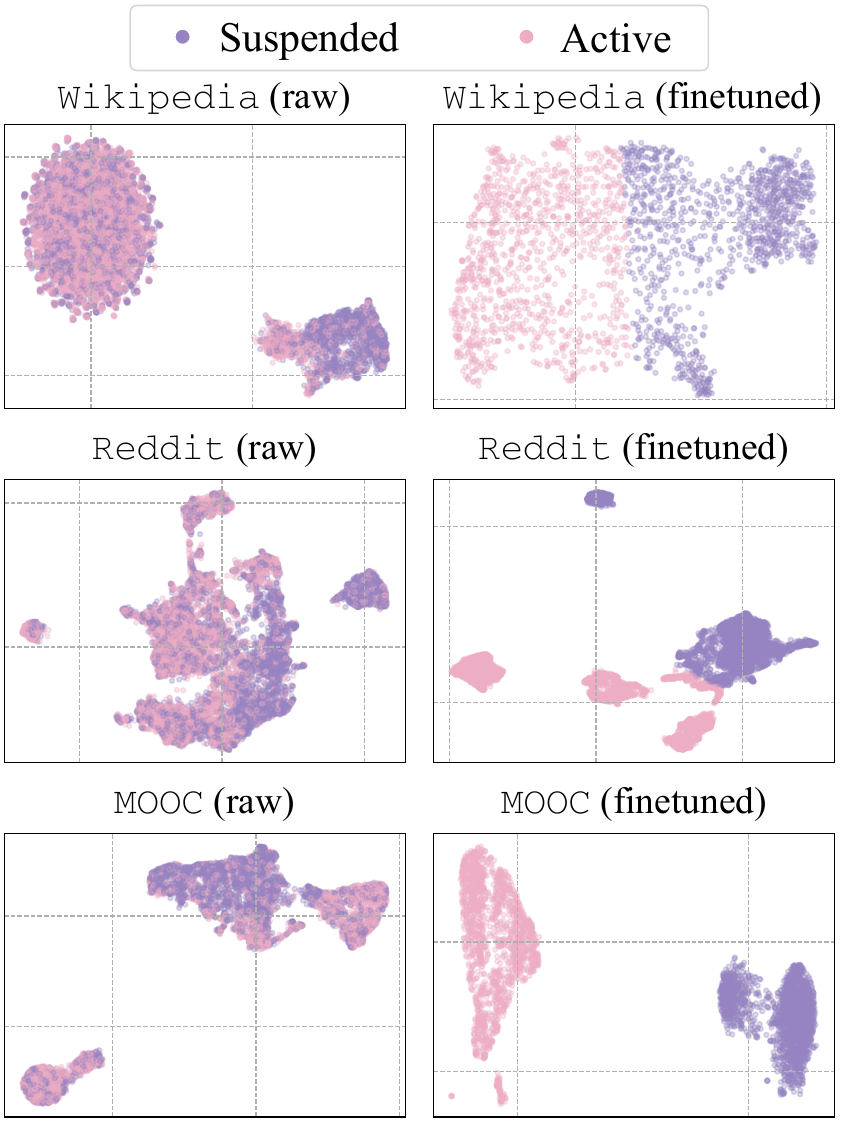}
        \caption{User Node Visualizations.}
        \label{fig:embedding}
    \end{minipage}
\end{figure*}

\subsection{\textit{RQ3:} Fine-tuning Efficiency}
\label{sec:rq3}

To test fine-tuning efficiency, we compare \modelname~with two competitive baselines on \texttt{Reddit} for inductive link prediction.

Results in Figure~\ref{fig:curve} show \modelname~converges the fastest, stabilizing within 56 epochs, compared with 216 epochs for \texttt{DyGPrompt} and 285 epochs for \texttt{TIGPrompt}. This indicates that divergence-conditioned prompting provides a more effective initialization for downstream adaptation.
\modelname~also achieves the highest AUC-ROC of 80.6\%, with its smoother convergence curves and fewer fluctuations. It suggests that semantic-temporal priors help guide optimization more stably during fine-tuning.
In terms of GPU memory usage, \modelname~falls between the two baselines: it consumes slightly more memory than \texttt{DyGPrompt}, but substantially less than \texttt{TIGPrompt} (approximately 15 GB, 14 GB, and 18 GB). Considering its higher predictive performance, \modelname~achieves the best accuracy-memory trade-off.

\textbf{Key Takeaway:}
\modelname~achieves the fastest convergence and highest task performance with moderate GPU usage, demonstrating superior fine-tuning efficiency and scalability.



\subsection{\textit{RQ4:} Interpretability of Routing Weights}
\label{sec:rq4}

To examine whether the learned routing weights are interpretable and informative, we randomly sample 100 nodes from each target domain and visualize the normalized routing attention toward different source prototypes for node classification.

The results are shown in Figure~\ref{fig:heat}, where columns denote source domains (\texttt{W}: \texttt{Wikipedia}, \texttt{R}: \texttt{Reddit}, \texttt{M}: \texttt{MOOC}, \texttt{G}: \texttt{Genre}), and each row corresponds to the routing distribution of one sampled node. In the ASDA setting, nodes consistently assign the highest routing weights to their own-domain prototypes, suggesting that the router effectively captures domain-aligned semantics and reuses intra-domain knowledge. In the LODO setting, routing patterns become softer and more distributed. Unseen target domains activate multiple source experts, such as \texttt{Genre} nodes attending to both \texttt{Reddit} and \texttt{MOOC}, demonstrating adaptive knowledge blending across related domains.

\textbf{Key Takeaway:}
The routing mechanism is interpretable and adaptive. When the target domain is seen during pre-training, it assigns higher attention to domain-specific prototypes for efficient knowledge reuse. When the target domain is unseen, it flexibly redistributes weights across semantically related source prototypes, revealing a balance between specialization and generalization.

\subsection{\textit{RQ5:} Evolution of Node Embeddings}
\label{sec:rq5}

To examine how node representations evolve during fine-tuning, we visualize user embeddings before and after adaptation using UMAP~\cite{mcinnes2018umap} on three representative user interaction datasets. These datasets contain dynamic labels indicating whether each user is Active or Suspended, which allows us to analyze the separation of behavioral states.

Results in Figure~\ref{fig:embedding} show that raw embeddings of users with different activity states are highly entangled. This suggests that pre-training mainly captures general structural and temporal regularities, rather than task-specific discriminative semantics. After fine-tuning, the embeddings of the two user types become more clearly separable, especially on \texttt{Wikipedia} and \texttt{MOOC}, where distinct cluster patterns emerge.

\textbf{Key Takeaway:}
Fine-tuning reorganizes the embedding space into more discriminative regions while preserving continuity with the pre-trained representations, showing that \modelname~enhances the separability of user states without losing general semantic information.

\subsection{\textit{RQ6:} Hyper-parameter Sensitivity Analysis}
\label{sec:rq6}

We conduct analysis on the sensitivity of three key hyper-parameters: $\lambda_\text{s}$ for semantic divergence in Eq.~\eqref{eq:div_comb}, $\lambda_\text{t}$ for temporal divergence in Eq.~\eqref{eq:div_comb}, and $\lambda_r$ for routing regularization in Eq.~\eqref{eq:ftn_overall}.

Results in Figure~\ref{fig:sensitivity} show that $\lambda_\text{s}$ has a moderate influence on performance, with the best accuracy achieved around $\lambda_\text{s}=\text{0.6}$. Smaller values may underemphasize semantic alignment, while larger values slightly over-regularize the semantic divergence term.
For $\lambda_\text{t}$, the model maintains stable performance within the range of 0.1 to 0.3, but performance decreases when the value becomes overly large. This indicates that moderate temporal regularization helps preserve temporal smoothness, whereas excessive regularization may suppress useful dynamic patterns.
For $\lambda_r$, performance shows a sharp optimum near 0.1, suggesting that routing regularization needs to balance expert flexibility and stability. A small $\lambda_r$ may lead to unstable expert activation, while an overly large value may constrain target-domain adaptation.

\textbf{Key Takeaway:}
\modelname~is robust to a broad range of hyper-parameter values, maintaining stable performance under all the three factors. This indicates that its performance does not rely heavily on finely tuned hyper-parameter choices.

\begin{figure}[t]
    \centering
    \includegraphics[width=0.325\textwidth]{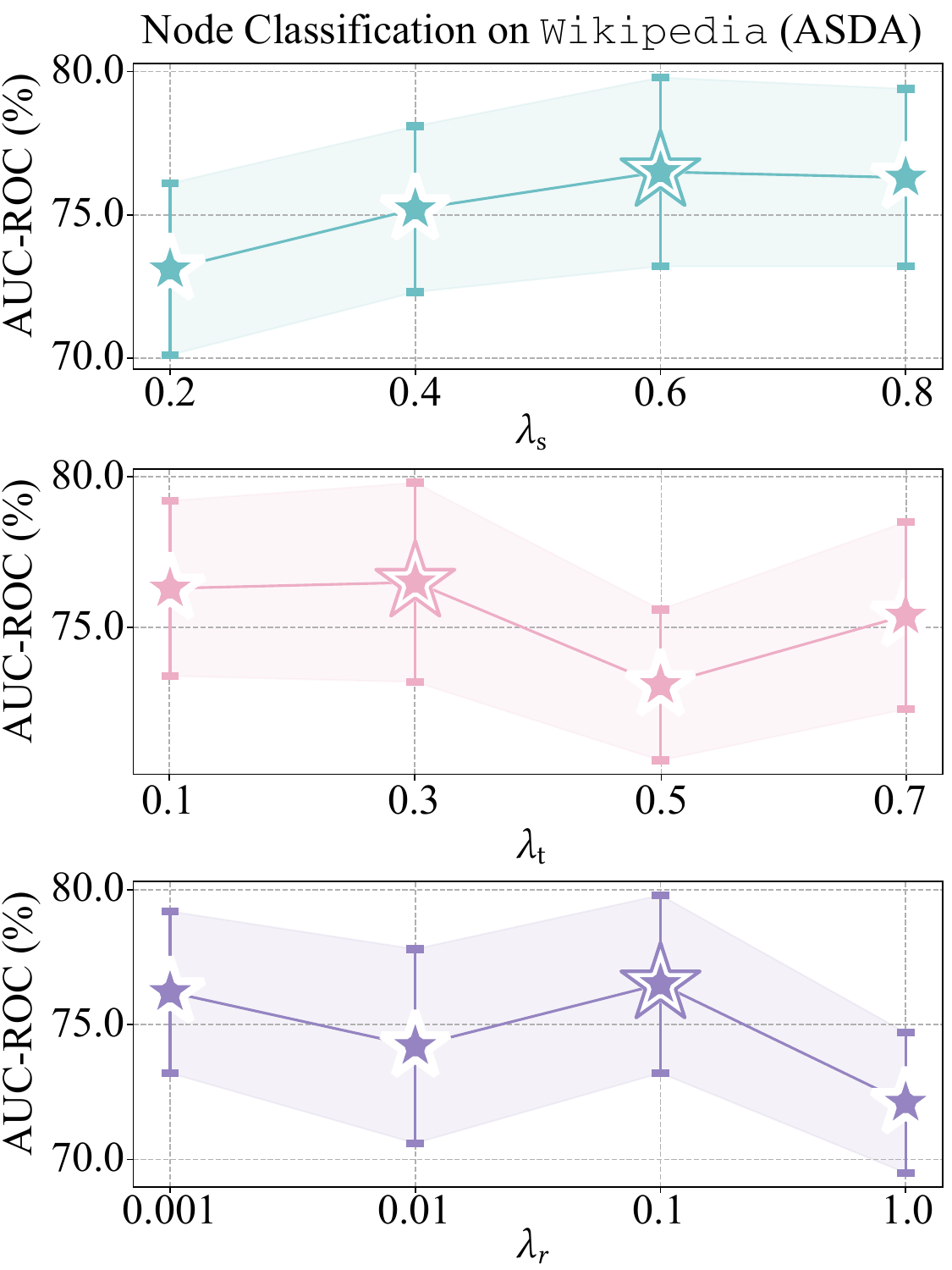}
    \caption{Hyper-parameter Sensitivity Analysis.}
    \label{fig:sensitivity}
\end{figure}

%% file: table/dataset.tex
\begin{table}[t]
  \centering
  \setlength{\tabcolsep}{2pt}
  \caption{Statistics of the dynamic graph dataset.}
  \resizebox{0.49\textwidth}{!}{
    \begin{tabular}{lrrrrrr}
    \toprule
    \textbf{Dataset} & $\#$\textbf{Node} & $\#$\textbf{Edge} & $\#$\textbf{Class} & $\#$\textbf{Label} & \makecell[r]{$\#$\textbf{Feat.}\\ \textbf{Dim.}} & \makecell[r]{\textbf{Time}\\\textbf{Span}} \\
    \midrule
    \texttt{Wikipedia}~\cite{kumar2019predicting} & 9,227 & 157,474 & 2 & 217 & 172 & 30 d\\
    \texttt{Reddit}~\cite{kumar2019predicting} & 11,000 & 672,447 & 2 & 366 & 172 & 30 d\\
    \texttt{MOOC}~\cite{kumar2019predicting} & 7,144 & 411,749 & 2 & 4,066 & 172 & 30 d\\
    \texttt{Genre}~\cite{huang2023temporal} & 1,505 & 17,858,395 & 474 & 984 & 86 & 1,500 d\\
    \bottomrule
    \end{tabular}%
    }
  \label{tab:dataset}%
\end{table}%

%% file: table/res_main.tex
\begin{table*}[!t]
\setlength{\tabcolsep}{2.5pt}
  \centering
  \caption{Comparison of node classification and link prediction (transductive and inductive) performance under ASDA and LODO settings. Results are reported with AUC-ROC (\%) with mean ± standard deviation over five runs. ``{\normalfont{bb.}}'' denotes backbones.}
  \resizebox{\textwidth}{!}{
    \begin{tabular}{lcccccccccccc}
    \toprule
    \textbf{Task} & \multicolumn{4}{c}{\textbf{Node Classification}} & \multicolumn{4}{c}{\textbf{Link Prediction} {(transductive)}} & \multicolumn{4}{c}{\textbf{Link Prediction} {(inductive)}} \\
    \cmidrule(r{1.5mm}){1-1} \cmidrule(r{1mm}){2-5} \cmidrule(r{1mm}){6-9} \cmidrule{10-13}
    \textbf{Target} & \texttt{{Wikipedia}} & \texttt{{Reddit}} & \texttt{{MOOC}}  & \texttt{{Genre}} & \texttt{{Wikipedia}} & \texttt{{Reddit}} & \texttt{{MOOC}}  & \texttt{{Genre}} & \texttt{{Wikipedia}} & \texttt{{Reddit}} & \texttt{{MOOC}}  & \texttt{{Genre}} \\
    \midrule[0.7pt]
    \textbf{Methods / Setting} & \multicolumn{12}{c}{\textbf{ASDA} (All-Seen-Domain-Adaptation)} \\
    \midrule[0.7pt]
    \texttt{TGN}~\scalebox{0.7}{\textcolor{Gray}{\textit{arXiv'20}}}~\cite{rossi2020temporal}   & 49.5\scalebox{0.8}{\;±5.1} & 47.6\scalebox{0.8}{\;±3.9} & 47.2\scalebox{0.8}{\;±7.3} & 47.9\scalebox{0.8}{\;±8.2} & 67.3\scalebox{0.8}{\;±4.1} & 70.2\scalebox{0.8}{\;±3.3} & 52.6\scalebox{0.8}{\;±3.1} & 64.5\scalebox{0.8}{\;±3.5} & 57.2\scalebox{0.8}{\;±2.5} & 67.0\scalebox{0.8}{\;±1.8} & 51.3\scalebox{0.8}{\;±7.7} & 56.7\scalebox{0.8}{\;±3.1} \\
    \texttt{TGAT}~\scalebox{0.7}{\textcolor{Gray}{\textit{ICLR'20}}} (bb.)~\cite{xu2020inductive}  & 65.3\scalebox{0.8}{\;±5.3} & 50.0\scalebox{0.8}{\;±4.5} & 55.3\scalebox{0.8}{\;±6.6} & 48.5\scalebox{0.8}{\;±6.6} & 54.1\scalebox{0.8}{\;±1.1} & 60.9\scalebox{0.8}{\;±2.1} & 53.6\scalebox{0.8}{\;±1.4} & 62.8\scalebox{0.8}{\;±5.1} & 52.9\scalebox{0.8}{\;±3.9} & 59.5\scalebox{0.8}{\;±2.2} & 52.8\scalebox{0.8}{\;±2.5} & 55.4\scalebox{0.8}{\;±1.9} \\
    \texttt{ROLAND-GCN}~\scalebox{0.7}{\textcolor{Gray}{\textit{KDD'22}}}~\cite{you2022roland} & 56.2\scalebox{0.8}{\;±5.6} & 45.5\scalebox{0.8}{\;±9.0} & 46.4\scalebox{0.8}{\;±8.1} & 43.5\scalebox{0.8}{\;±3.9} & 51.4\scalebox{0.8}{\;±1.1} & 57.5\scalebox{0.8}{\;±2.4} & 52.6\scalebox{0.8}{\;±1.8} & 53.7\scalebox{0.8}{\;±3.7} & 51.5\scalebox{0.8}{\;±1.2} & 55.7\scalebox{0.8}{\;±3.7} & 51.7\scalebox{0.8}{\;±4.5} & 53.7\scalebox{0.8}{\;±3.3} \\
    \texttt{ROLAND-GAT}~\scalebox{0.7}{\textcolor{Gray}{\textit{KDD'22}}}~\cite{you2022roland} & 60.0\scalebox{0.8}{\;±6.6} & 46.8\scalebox{0.8}{\;±8.0} & 47.2\scalebox{0.8}{\;±4.8} & 44.1\scalebox{0.8}{\;±3.9} & 54.6\scalebox{0.8}{\;±2.1} & 59.0\scalebox{0.8}{\;±2.4} & 52.5\scalebox{0.8}{\;±1.6} & 54.1\scalebox{0.8}{\;±1.3} & 54.3\scalebox{0.8}{\;±2.3} & 57.4\scalebox{0.8}{\;±2.5} & 52.2\scalebox{0.8}{\;±3.4} & 54.3\scalebox{0.8}{\;±2.2} \\
    \texttt{TREND}~\scalebox{0.7}{\textcolor{Gray}{\textit{WWW'22}}}~\cite{wen2022trend} & 67.0\scalebox{0.8}{\;±6.7} & 62.3\scalebox{0.8}{\;±6.7} & 63.3\scalebox{0.8}{\;±4.0} & 47.5\scalebox{0.8}{\;±5.3} & 61.9\scalebox{0.8}{\;±2.2} & 77.6\scalebox{0.8}{\;±1.1} & 54.8\scalebox{0.8}{\;±3.1} & 52.5\scalebox{0.8}{\;±2.5} & 52.5\scalebox{0.8}{\;±4.5} & 66.4\scalebox{0.8}{\;±2.8} & 52.3\scalebox{0.8}{\;±2.8} & 51.2\scalebox{0.8}{\;±1.0} \\
    \cmidrule(r{1.5mm}){1-1} \cmidrule(r{1mm}){2-5} \cmidrule(r{1mm}){6-9} \cmidrule{10-13}
    \texttt{DDGCL}~\scalebox{0.7}{\textcolor{Gray}{\textit{CIKM'21}}}~\cite{tian2021self} & 63.8\scalebox{0.8}{\;±7.4} & 52.0\scalebox{0.8}{\;±6.4} & 59.3\scalebox{0.8}{\;±9.5} & 46.5\scalebox{0.8}{\;±7.7} & 52.7\scalebox{0.8}{\;±4.2} & 60.9\scalebox{0.8}{\;±2.5} & 52.7\scalebox{0.8}{\;±5.5} & 63.5\scalebox{0.8}{\;±4.8} & 50.7\scalebox{0.8}{\;±3.6} & 54.7\scalebox{0.8}{\;±1.2} & 51.3\scalebox{0.8}{\;±1.7} & 50.0\scalebox{0.8}{\;±3.3} \\
    \texttt{CPDG}~\scalebox{0.7}{\textcolor{Gray}{\textit{ICDE'24}}}~\cite{bei2024cpdg}  & 49.2\scalebox{0.8}{\;±4.3} & 58.7\scalebox{0.8}{\;±4.7} & 51.3\scalebox{0.8}{\;±8.6} & 44.8\scalebox{0.8}{\;±5.9} & 53.9\scalebox{0.8}{\;±2.7} & 59.6\scalebox{0.8}{\;±1.6} & 53.3\scalebox{0.8}{\;±1.9} & 51.1\scalebox{0.8}{\;±2.1} & 50.7\scalebox{0.8}{\;±3.7} & 57.1\scalebox{0.8}{\;±3.3} & 52.5\scalebox{0.8}{\;±3.9} & 51.0\scalebox{0.8}{\;±4.4} \\
    \cmidrule(r{1.5mm}){1-1} \cmidrule(r{1mm}){2-5} \cmidrule(r{1mm}){6-9} \cmidrule{10-13}
    \texttt{GraphPrompt}~\scalebox{0.7}{\textcolor{Gray}{\textit{WWW'23}}}~\cite{liu2023graphprompt} & 71.0\scalebox{0.8}{\;±4.5} & 57.5\scalebox{0.8}{\;±4.3} & 62.3\scalebox{0.8}{\;±5.3} & 47.0\scalebox{0.8}{\;±5.1} & 53.5\scalebox{0.8}{\;±1.1} & 65.8\scalebox{0.8}{\;±1.7} & 50.7\scalebox{0.8}{\;±5.7} & 64.5\scalebox{0.8}{\;±6.0} & 52.1\scalebox{0.8}{\;±2.4} & 61.6\scalebox{0.8}{\;±2.6} & 50.3\scalebox{0.8}{\;±3.6} & 57.4\scalebox{0.8}{\;±1.3} \\
    \texttt{ProG}~\scalebox{0.7}{\textcolor{Gray}{\textit{NeurIPS'24}}}~\cite{zi2024prog}  & 57.5\scalebox{0.8}{\;±2.6} & 65.8\scalebox{0.8}{\;±3.4} & 60.3\scalebox{0.8}{\;±2.2} & 46.0\scalebox{0.8}{\;±2.2} & 65.9\scalebox{0.8}{\;±2.8} & 88.9\scalebox{0.8}{\;±2.2} & 54.2\scalebox{0.8}{\;±3.6} & 64.1\scalebox{0.8}{\;±3.1} & 57.9\scalebox{0.8}{\;±1.1} & 76.7\scalebox{0.8}{\;±5.7} & 53.7\scalebox{0.8}{\;±6.6} & 55.9\scalebox{0.8}{\;±2.6} \\
    \texttt{GCOPE}~\scalebox{0.7}{\textcolor{Gray}{\textit{KDD'24}}}~\cite{zhao2024all} & 59.9\scalebox{0.8}{\;±2.7} & 66.0\scalebox{0.8}{\;±6.6} & 66.0\scalebox{0.8}{\;±4.7} & 49.2\scalebox{0.8}{\;±5.4} & 66.5\scalebox{0.8}{\;±5.3} & 90.2\scalebox{0.8}{\;±1.9} & 53.6\scalebox{0.8}{\;±4.9} & 65.1\scalebox{0.8}{\;±6.0} & \underline{58.7\scalebox{0.8}{\;±3.1}} & 75.1\scalebox{0.8}{\;±2.7} & 52.1\scalebox{0.8}{\;±2.9} & 54.0\scalebox{0.8}{\;±3.8} \\
    \texttt{SAMGPT}~\scalebox{0.7}{\textcolor{Gray}{\textit{WWW'25}}}~\cite{yu2025samgpt} & 65.0\scalebox{0.8}{\;±4.0} & \underline{74.2\scalebox{0.8}{\;±2.2}} & 73.8\scalebox{0.8}{\;±5.4} & 50.1\scalebox{0.8}{\;±3.6} & \underline{69.9\scalebox{0.8}{\;±3.9}} & 90.3\scalebox{0.8}{\;±3.7} & 53.4\scalebox{0.8}{\;±3.7} & 65.9\scalebox{0.8}{\;±3.6} & 57.4\scalebox{0.8}{\;±4.1} & \underline{78.4\scalebox{0.8}{\;±2.0}} & 52.7\scalebox{0.8}{\;±1.3} & 54.9\scalebox{0.8}{\;±2.4} \\
    \cmidrule(r{1.5mm}){1-1} \cmidrule(r{1mm}){2-5} \cmidrule(r{1mm}){6-9} \cmidrule{10-13}
    \texttt{TIGPrompt}~\scalebox{0.7}{\textcolor{Gray}{\textit{arXiv'24}}}~\cite{chen2024prompt} & 69.8\scalebox{0.8}{\;±3.9} & 65.2\scalebox{0.8}{\;±5.9} & 70.4\scalebox{0.8}{\;±4.9} & 49.1\scalebox{0.8}{\;±3.9} & 60.8\scalebox{0.8}{\;±2.8} & 86.2\scalebox{0.8}{\;±1.8} & 51.6\scalebox{0.8}{\;±2.7} & 66.2\scalebox{0.8}{\;±2.9} & 51.7\scalebox{0.8}{\;±1.2} & 64.1\scalebox{0.8}{\;±2.4} & 51.8\scalebox{0.8}{\;±4.9} & 52.7\scalebox{0.8}{\;±3.9} \\
    \texttt{DyGPrompt}~\scalebox{0.7}{\textcolor{Gray}{\textit{ICLR'25}}}~\cite{yu2025nodetime} & \underline{74.0\scalebox{0.8}{\;±3.2}} & 72.8\scalebox{0.8}{\;±3.7} & \underline{76.0\scalebox{0.8}{\;±2.7}} & \underline{52.2\scalebox{0.8}{\;±5.2}} & 67.5\scalebox{0.8}{\;±4.7} & \underline{92.7\scalebox{0.8}{\;±2.9}} & \underline{55.3\scalebox{0.8}{\;±2.1}} & \underline{66.5\scalebox{0.8}{\;±2.4}} & 57.6\scalebox{0.8}{\;±2.4} & 74.5\scalebox{0.8}{\;±2.2} & \underline{53.8\scalebox{0.8}{\;±3.3}} & \underline{57.5\scalebox{0.8}{\;±2.3}} \\
    \cmidrule(r{1.5mm}){1-1} \cmidrule(r{1mm}){2-5} \cmidrule(r{1mm}){6-9} \cmidrule{10-13}
    \textbf{\modelname~\textit{(ours)}} & \textbf{76.5\scalebox{0.8}{\;±3.3}} & \textbf{77.5\scalebox{0.8}{\;±4.3}} & \textbf{79.8\scalebox{0.8}{\;±6.0}} & \textbf{55.3\scalebox{0.8}{\;±4.2}} & \textbf{70.5\scalebox{0.8}{\;±1.2}} & \textbf{94.2\scalebox{0.8}{\;±0.6}} & \textbf{55.4\scalebox{0.8}{\;±3.7}} & \textbf{67.1\scalebox{0.8}{\;±3.4}} & \textbf{59.6\scalebox{0.8}{\;±0.6}} & \textbf{80.6\scalebox{0.8}{\;±0.3}} & \textbf{54.0\scalebox{0.8}{\;±4.5}} & \textbf{58.7\scalebox{0.8}{\;±4.6}} \\
    \midrule[0.7pt]
    \textbf{Methods / Setting} & \multicolumn{12}{c}{\textbf{LODO} (Leave-One-Domain-Out)} \\
    \midrule[0.7pt]
    \texttt{TGN}~\scalebox{0.7}{\textcolor{Gray}{\textit{arXiv'20}}}~\cite{rossi2020temporal}   & 54.2\scalebox{0.8}{\;±6.5} & 44.8\scalebox{0.8}{\;±3.7} & 44.2\scalebox{0.8}{\;±7.0} & 44.8\scalebox{0.8}{\;±9.6} & 59.1\scalebox{0.8}{\;±2.0} & 63.3\scalebox{0.8}{\;±3.1} & 51.9\scalebox{0.8}{\;±4.4} & 62.7\scalebox{0.8}{\;±2.4} & 54.5\scalebox{0.8}{\;±2.1} & 65.0\scalebox{0.8}{\;±2.8} & 53.6\scalebox{0.8}{\;±6.9} & 56.2\scalebox{0.8}{\;±1.3} \\
    \texttt{TGAT}~\scalebox{0.7}{\textcolor{Gray}{\textit{ICLR'20}}} (bb.)~\cite{xu2020inductive}  & 61.8\scalebox{0.8}{\;±6.1} & 55.4\scalebox{0.8}{\;±3.8} & 52.3\scalebox{0.8}{\;±6.4} & 45.1\scalebox{0.8}{\;±7.1} & 52.4\scalebox{0.8}{\;±1.3} & 58.2\scalebox{0.8}{\;±2.5} & 51.6\scalebox{0.8}{\;±1.4} & 60.7\scalebox{0.8}{\;±6.7} & 51.9\scalebox{0.8}{\;±4.3} & 57.8\scalebox{0.8}{\;±3.1} & 54.5\scalebox{0.8}{\;±2.7} & 55.4\scalebox{0.8}{\;±2.5} \\
    \texttt{ROLAND-GCN}~\scalebox{0.7}{\textcolor{Gray}{\textit{KDD'22}}}~\cite{you2022roland} & 53.5\scalebox{0.8}{\;±4.4} & 42.2\scalebox{0.8}{\;±9.2} & 43.3\scalebox{0.8}{\;±7.5} & 41.3\scalebox{0.8}{\;±2.3} & 50.7\scalebox{0.8}{\;±2.9} & 54.0\scalebox{0.8}{\;±2.1} & 51.3\scalebox{0.8}{\;±5.3} & 54.0\scalebox{0.8}{\;±2.2} & 50.1\scalebox{0.8}{\;±2.4} & 54.5\scalebox{0.8}{\;±2.8} & 53.9\scalebox{0.8}{\;±3.7} & 52.0\scalebox{0.8}{\;±4.0} \\
    \texttt{ROLAND-GAT}~\scalebox{0.7}{\textcolor{Gray}{\textit{KDD'22}}}~\cite{you2022roland} & 57.5\scalebox{0.8}{\;±6.3} & 42.2\scalebox{0.8}{\;±7.5} & 44.7\scalebox{0.8}{\;±5.7} & 42.6\scalebox{0.8}{\;±1.7} & 53.9\scalebox{0.8}{\;±3.4} & 56.9\scalebox{0.8}{\;±1.7} & 50.6\scalebox{0.8}{\;±3.5} & 55.0\scalebox{0.8}{\;±3.1} & 52.4\scalebox{0.8}{\;±3.1} & 55.4\scalebox{0.8}{\;±1.8} & 53.7\scalebox{0.8}{\;±1.0} & 54.0\scalebox{0.8}{\;±1.3} \\
    \texttt{TREND}~\scalebox{0.7}{\textcolor{Gray}{\textit{WWW'22}}}~\cite{wen2022trend} & 63.1\scalebox{0.8}{\;±5.5} & 59.1\scalebox{0.8}{\;±5.1} & 61.3\scalebox{0.8}{\;±2.2} & 46.0\scalebox{0.8}{\;±3.2} & 60.2\scalebox{0.8}{\;±1.9} & 74.3\scalebox{0.8}{\;±1.6} & 51.0\scalebox{0.8}{\;±5.5} & 53.5\scalebox{0.8}{\;±2.3} & 51.2\scalebox{0.8}{\;±5.5} & 64.3\scalebox{0.8}{\;±2.1} & 55.3\scalebox{0.8}{\;±3.0} & 52.4\scalebox{0.8}{\;±1.3} \\
    \cmidrule(r{1.5mm}){1-1} \cmidrule(r{1mm}){2-5} \cmidrule(r{1mm}){6-9} \cmidrule{10-13}
    \texttt{DDGCL}~\scalebox{0.7}{\textcolor{Gray}{\textit{CIKM'21}}}~\cite{tian2021self} & 60.6\scalebox{0.8}{\;±8.0} & 50.6\scalebox{0.8}{\;±8.8} & 56.0\scalebox{0.8}{\;±8.7} & 44.5\scalebox{0.8}{\;±5.5} & 52.0\scalebox{0.8}{\;±4.8} & 58.1\scalebox{0.8}{\;±1.9} & 52.2\scalebox{0.8}{\;±7.7} & 61.6\scalebox{0.8}{\;±5.9} & 50.2\scalebox{0.8}{\;±1.4} & 51.4\scalebox{0.8}{\;±1.3} & 53.2\scalebox{0.8}{\;±1.7} & 50.6\scalebox{0.8}{\;±3.6} \\
    \texttt{CPDG}~\scalebox{0.7}{\textcolor{Gray}{\textit{ICDE'24}}}~\cite{bei2024cpdg}  & 46.9\scalebox{0.8}{\;±3.5} & 56.0\scalebox{0.8}{\;±3.4} & 45.9\scalebox{0.8}{\;±7.3} & 42.1\scalebox{0.8}{\;±6.4} & 51.3\scalebox{0.8}{\;±1.9} & 53.4\scalebox{0.8}{\;±1.4} & 50.9\scalebox{0.8}{\;±2.6} & 51.8\scalebox{0.8}{\;±2.4} & 50.9\scalebox{0.8}{\;±3.3} & 53.4\scalebox{0.8}{\;±1.2} & 54.7\scalebox{0.8}{\;±2.8} & 50.5\scalebox{0.8}{\;±2.6} \\
    \cmidrule(r{1.5mm}){1-1} \cmidrule(r{1mm}){2-5} \cmidrule(r{1mm}){6-9} \cmidrule{10-13}
    \texttt{GraphPrompt}~\scalebox{0.7}{\textcolor{Gray}{\textit{WWW'23}}}~\cite{liu2023graphprompt} & 68.2\scalebox{0.8}{\;±3.0} & 53.7\scalebox{0.8}{\;±3.5} & 58.1\scalebox{0.8}{\;±3.5} & 43.9\scalebox{0.8}{\;±4.8} & 52.3\scalebox{0.8}{\;±1.1} & 62.4\scalebox{0.8}{\;±1.3} & 52.2\scalebox{0.8}{\;±5.1} & 62.4\scalebox{0.8}{\;±5.6} & 51.9\scalebox{0.8}{\;±1.2} & 58.6\scalebox{0.8}{\;±2.0} & 51.7\scalebox{0.8}{\;±5.6} & 56.6\scalebox{0.8}{\;±1.8} \\
    \texttt{ProG}~\scalebox{0.7}{\textcolor{Gray}{\textit{NeurIPS'24}}}~\cite{zi2024prog}  & 53.8\scalebox{0.8}{\;±4.3} & 62.4\scalebox{0.8}{\;±2.0} & 58.4\scalebox{0.8}{\;±4.6} & 43.8\scalebox{0.8}{\;±2.9} & 63.4\scalebox{0.8}{\;±3.3} & 86.5\scalebox{0.8}{\;±4.7} & 50.2\scalebox{0.8}{\;±4.5} & 61.4\scalebox{0.8}{\;±1.9} & 55.4\scalebox{0.8}{\;±3.9} & 73.0\scalebox{0.8}{\;±3.3} & 53.9\scalebox{0.8}{\;±4.4} & 56.9\scalebox{0.8}{\;±1.7} \\
    \texttt{GCOPE}~\scalebox{0.7}{\textcolor{Gray}{\textit{KDD'24}}}~\cite{zhao2024all} & 58.0\scalebox{0.8}{\;±2.4} & 64.7\scalebox{0.8}{\;±6.0} & 62.2\scalebox{0.8}{\;±4.2} & 49.7\scalebox{0.8}{\;±4.9} & 67.6\scalebox{0.8}{\;±5.2} & 86.5\scalebox{0.8}{\;±2.5} & 52.2\scalebox{0.8}{\;±2.0} & 64.8\scalebox{0.8}{\;±4.6} & 54.2\scalebox{0.8}{\;±1.8} & 71.3\scalebox{0.8}{\;±3.6} & 52.3\scalebox{0.8}{\;±1.5} & 55.0\scalebox{0.8}{\;±2.1} \\
    \texttt{SAMGPT}~\scalebox{0.7}{\textcolor{Gray}{\textit{WWW'25}}}~\cite{yu2025samgpt} & 70.0\scalebox{0.8}{\;±2.4} & 72.9\scalebox{0.8}{\;±4.1} & 72.3\scalebox{0.8}{\;±7.5} & \underline{51.1\scalebox{0.8}{\;±4.1}} & \underline{68.0\scalebox{0.8}{\;±2.7}} & 86.3\scalebox{0.8}{\;±1.2} & 52.6\scalebox{0.8}{\;±1.5} & \underline{66.5\scalebox{0.8}{\;±1.3}} & 55.5\scalebox{0.8}{\;±5.8} & 73.0\scalebox{0.8}{\;±1.4} & \underline{55.6\scalebox{0.8}{\;±3.9}} & \textbf{58.4\scalebox{0.8}{\;±1.0}} \\
    \cmidrule(r{1.5mm}){1-1} \cmidrule(r{1mm}){2-5} \cmidrule(r{1mm}){6-9} \cmidrule{10-13}
    \texttt{TIGPrompt}~\scalebox{0.7}{\textcolor{Gray}{\textit{arXiv'24}}}~\cite{chen2024prompt} & 67.2\scalebox{0.8}{\;±4.2} & 62.9\scalebox{0.8}{\;±5.0} & 68.5\scalebox{0.8}{\;±5.8} & 48.0\scalebox{0.8}{\;±4.4} & 57.5\scalebox{0.8}{\;±1.6} & 74.6\scalebox{0.8}{\;±3.0} & 51.3\scalebox{0.8}{\;±1.1} & 64.2\scalebox{0.8}{\;±4.0} & 50.1\scalebox{0.8}{\;±4.0} & 62.8\scalebox{0.8}{\;±1.4} & 52.7\scalebox{0.8}{\;±3.6} & 55.3\scalebox{0.8}{\;±3.1} \\
    \texttt{DyGPrompt}~\scalebox{0.7}{\textcolor{Gray}{\textit{ICLR'25}}}~\cite{yu2025nodetime} & \underline{74.6\scalebox{0.8}{\;±5.9}} & \underline{74.5\scalebox{0.8}{\;±4.6}} & \underline{72.8\scalebox{0.8}{\;±3.8}} & 47.1\scalebox{0.8}{\;±4.1} & 64.8\scalebox{0.8}{\;±5.7} & \underline{87.6\scalebox{0.8}{\;±1.0}} & \underline{53.0\scalebox{0.8}{\;±4.8}} & 65.3\scalebox{0.8}{\;±2.3} & \underline{55.8\scalebox{0.8}{\;±2.1}} & \underline{74.0\scalebox{0.8}{\;±4.7}} & 54.7\scalebox{0.8}{\;±3.2} & 57.7\scalebox{0.8}{\;±2.3} \\
    \cmidrule(r{1.5mm}){1-1} \cmidrule(r{1mm}){2-5} \cmidrule(r{1mm}){6-9} \cmidrule{10-13}
    \textbf{\modelname~\textit{(ours)}} & \textbf{76.7\scalebox{0.8}{\;±6.2}} & \textbf{76.8\scalebox{0.8}{\;±4.7}} & \textbf{78.6\scalebox{0.8}{\;±2.0}} & \textbf{52.6\scalebox{0.8}{\;±3.3}} & \textbf{70.2\scalebox{0.8}{\;±1.1}} & \textbf{93.8\scalebox{0.8}{\;±4.2}} & \textbf{53.1\scalebox{0.8}{\;±0.6}} & \textbf{67.4\scalebox{0.8}{\;±0.6}} & \textbf{56.1\scalebox{0.8}{\;±0.5}} & \textbf{74.7\scalebox{0.8}{\;±1.1}} & \textbf{56.0\scalebox{0.8}{\;±4.1}} & \underline{58.0\scalebox{0.8}{\;±5.4}} \\
    \bottomrule
    \end{tabular}%
}
\label{tab:main}
\end{table*}

%% file: table/res_klodo.tex
\begin{table}[t]
  \centering
  \caption{Results under the ${K}$-LODO setting for node classification with AUC-ROC (\%). ``\texttt{W}'' denotes {\textnormal{\texttt{Wikipedia}}}, ``\texttt{R}'' denotes {\textnormal{\texttt{Reddit}}}, ``\texttt{M}'' denotes {\textnormal{\texttt{MOOC}}}, and ``\texttt{G}'' denotes {\textnormal{\texttt{Genre}}}.}
  \resizebox{\linewidth}{!}{
    \centering
    \begin{tabular}{l|lcccc}
    \toprule
    $\boldsymbol{K}$     & \textbf{Source}     & \makecell{\texttt{R}\;+\;\texttt{M}\\\texttt{R}\;+\;\texttt{G}\\\texttt{M}\;+\;\texttt{G}}    & \makecell{\texttt{W}\;+\;\texttt{M}\\\texttt{W}\;+\;\texttt{G}\\\texttt{M}\;+\;\texttt{G}}    & \makecell{\texttt{W}\;+\;\texttt{R}\\\texttt{W}\;+\;\texttt{G}\\\texttt{R}\;+\;\texttt{G}}    & \makecell{\texttt{W}\;+\;\texttt{R}\\\texttt{W}\;+\;\texttt{M}\\\texttt{R}\;+\;\texttt{M}} \\
    \midrule
    \multirow{3}[10]{*} {\makecell[l]{2}}     & \textbf{Target}     & \texttt{Wikipedia}     & \texttt{Reddit}     & \texttt{MOOC}     & \texttt{Genre} \\
    \cmidrule{2-6}
    & \textbf{\modelname} & \makecell{77.6\\75.4\\78.1}    & \makecell{76.9\\77.6\\78.2}    & \makecell{74.8\\78.4\\78.9}    & \makecell{51.0\\54.2\\53.3} \\
    \cmidrule{2-6}
    & \textbf{Average} & 77.0\scalebox{0.7}{\;±1.4}    & 77.6\scalebox{0.7}{\;±0.7}    & 77.4\scalebox{0.7}{\;±2.2}    & 52.8\scalebox{0.7}{\;±1.7} \\
    \midrule[0.7pt]
    $\boldsymbol{K}$     & \textbf{Source}     & \makecell{\texttt{R}\\\texttt{M}\\\texttt{G}}    & \makecell{\texttt{W}\\\texttt{M}\\\texttt{G}}    & \makecell{\texttt{W}\\\texttt{R}\\\texttt{G}}    & \makecell{\texttt{W}\\\texttt{R}\\\texttt{M}} \\
    \midrule
    \multirow{3}[10]{*} {\makecell[l]{3}}     & \textbf{Target}     & \texttt{Wikipedia}     & \texttt{Reddit}     & \texttt{MOOC}     & \texttt{Genre} \\
    \cmidrule{2-6}
    & \textbf{\modelname} & \makecell{79.4\\76.8\\75.2}    & \makecell{77.5\\75.1\\77.0}    & \makecell{78.9\\75.3\\79.6}    & \makecell{52.7\\52.1\\58.9} \\
    \cmidrule{2-6}
    & \textbf{Average} & 77.1\scalebox{0.7}{\;±2.1}    & 76.5\scalebox{0.7}{\;±1.3}    & 77.9\scalebox{0.7}{\;±2.3}    & 54.6\scalebox{0.7}{\;±3.8} \\
    \bottomrule
    \end{tabular}%
    }
  \label{tab:klodo}%
\end{table}%

%% file: 6_conclusion.tex
\section{Conclusion}
\label{sec:conclusion}
In this paper, we proposed \modelname, a dynamic Graph Foundation Model for multiple domains based on decoupled and divergence-conditioned prompting. \modelname~addresses the incompatibility of multi-domain dynamic graphs through semantic-temporal decoupled pre-training, where the semantic branch learns transferable feature semantics, and the temporal branch preserves domain-specific temporal dynamics. To mitigate negative transfer, \modelname~further introduces divergence-aware routing to select relevant source-domain knowledge, and employs divergence-conditioned prompts for efficient downstream fine-tuning with frozen encoders. Extensive experiments on four continuous-time dynamic graph benchmarks show that \modelname~outperforms representative DGNNs, dynamic graph pre-training methods, static GFMs, and dynamic graph prompting baselines on node classification and link prediction tasks, demonstrating strong transferability, adaptability, and fine-tuning efficiency.

%% file: appendix/1_notations.tex
\section{Summary of Notations}
For clarity, we summarize the main notations used throughout this paper in Table~\ref{tab:notation}. These notations cover the dynamic graph setting, semantic-temporal pre-training, divergence-aware routing, prompt-based fine-tuning, and complexity analysis.
\input{table/notation}

%% file: table/notation.tex
\begin{table}[htbp]
  \centering
  \setlength{\tabcolsep}{3pt}
  \renewcommand{\arraystretch}{1.2}
  \caption{Summary of notations.}
  \resizebox{0.49\textwidth}{!}{
\begin{tabular}{p{0.33\linewidth} p{0.63\linewidth}}

\toprule

\textbf{Notation} & \textbf{Description} \\

\midrule

\multicolumn{2}{l}{\textit{\textbf{Basic notations}}} \\

\midrule

$\mathcal{G}=\{\mathcal{V},\mathcal{E},\mathcal{T}\}$, $(u,v,t)$ & Dynamic graph and temporal interaction. \\

$\mathbf{X}$, $\mathbf{A}$, $\mathbf{e}_{uv}$ & Node features, adjacency, and edge features. \\

$d_0$, $d$, $d_t$, $r$, $p$ & Input, hidden, time, adapter, prompt dimensions. \\

$\mathbf{z}$, $\widetilde{\mathbf z}$, $\mathbf{h}$, $\mathbf{H}$ & Node-level and fused representations. \\

$\boldsymbol{\mathcal T}(\cdot)$, $\operatorname{PE}(\cdot)$, $\omega_i$ & Time encoding terms. \\

$\phi(\cdot)$, $\psi(\cdot)$, $\operatorname{AGG}(\cdot)$, $g(\cdot)$ & Message, ranking, aggregation, scoring functions. \\

\midrule

\multicolumn{2}{l}{\textit{\textbf{Pre-training and adaptation setting}}} \\

\midrule

$\{\mathcal{G}_i^{\mathcal S},D_i^{\mathcal S},Y_i^{\mathcal S}\}_{i=1}^n$ & Source graphs, domains, and labels. \\

$\mathcal{G}^{\mathcal T}$, $D^{\mathcal T}$, $\mathcal D_{\text{sup}}^{\mathcal T}$ & Target graph, domain, and support set. \\

$n$, $m$ & Number of domains and shots. \\

$h=g\circ f$, $f_\text{s}$, $f_\text{t}$ & Learner, semantic encoder, and temporal encoder. \\

$\boldsymbol{\theta}^{\star}$, $f_\text{s}^{\star}$, $f_\text{t}^{\star}$ & Frozen parameters and encoders. \\

\midrule

\multicolumn{2}{l}{\textit{\textbf{Semantic branch}}} \\

\midrule

$\mathbf{X}_i^{\mathcal S}$, $\mathbf{A}_i^{\mathcal S}$, $\widehat{\mathbf X}_i^{\mathcal S}$, $\widehat{\mathbf X}^{\mathcal T}$, $\widehat{\mathbf A}^{\mathcal T}$ & Source/target features and structures. \\

$\boldsymbol{\mathcal A}_i$, $\boldsymbol{\theta}_\text{s}$ & Aligner and semantic parameters. \\

$\mathbf{Z}^{\mathcal S}$, $\mathbf{Z}^{\mathcal S+}$, $\mathbf{Z}^{\mathcal T}$ & Source, positive, and target representations. \\

$I(\cdot;\cdot)$, $q_\phi(\cdot)$, $p(\cdot)$, $\operatorname{KL}(\cdot)$ & IB-related quantities. \\

$N^+$, $N^-$, $\tau$ & Positive count, negative count, and temperature. \\

\makecell[l]{$\mathcal L_\text{InfoNCE}$, $\mathcal L_\text{SS-IB}$, $\mathcal L_\text{pre}$,\\$\mathcal L_\text{pre-sem}$} & Semantic-branch losses. \\

$\lambda_1$ & IB regularization weight. \\

\midrule

\multicolumn{2}{l}{\textit{\textbf{Temporal branch}}} \\

\midrule

$\mathbf r_i^{\mathcal S}(t)$, $\mathbf r^{\mathcal T}(t)$, $\boldsymbol{\mathcal T}_i$ & Time embeddings and timer. \\

$\boldsymbol{\Gamma}_i$, $\boldsymbol{\Theta}_\text{t}$, $\boldsymbol{\Phi}_i$, $\boldsymbol{\Psi}_i$ & Adapter and temporal parameters. \\

$\mathcal N_i(v,t)$, $\mathbf m_{u\to v}^{\mathcal S}$, $a_{u\to v}$ & Temporal neighbors, messages, and attention. \\

$\mathcal B_i$, $\eta$ & Mini-batch and learning rate. \\

$\mathcal L_\text{pre-tem}^{(i)}$ & Temporal pre-training loss. \\

\midrule

\multicolumn{2}{l}{\textit{\textbf{Cross-domain routing}}} \\

\midrule

$\mathbf s_i^{\mathcal S}$, $\mathbf t_i^{\mathcal S}$, $\mathbf s^{\mathcal T}$, $\mathbf t^{\mathcal T}$ & Source and target prototypes. \\

$\overline{\mathcal G}^{\mathcal T}$, $\overline{\mathcal V}^{\mathcal T}$, $\overline{\mathcal E}^{\mathcal T}$ & Timestamp-removed target graph. \\

$\boldsymbol{\mu}^{\mathcal S}_i$, $\boldsymbol{\mu}^{\mathcal T}$, $\boldsymbol{\sigma}^{2\mathcal S}_i$, $\boldsymbol{\sigma}^{2\mathcal T}$ & Gaussian statistics. \\

$\operatorname{Div}_\text{s}$, $\operatorname{Div}_\text{t}$, $\operatorname{Div}_i$ & Semantic, temporal, and total divergences. \\

$\lambda_\text{s}$, $\lambda_\text{t}$, $\gamma$ & Routing hyper-parameters. \\

$\boldsymbol{\alpha}$, $\alpha_i$, $\bar{\mathbf s}^{\mathcal S}$, $\bar{\mathbf t}^{\mathcal S}$ & Routing weights and priors. \\

$\mathcal R_\text{routing}$, $H(\boldsymbol{\alpha})$, $\lambda_2$, $\lambda_3$ & Routing regularization terms. \\

\midrule

\multicolumn{2}{l}{\textit{\textbf{Prompt-based fine-tuning}}} \\

\midrule

$\boldsymbol{\mathcal P}_\text{s}$, $\boldsymbol{\mathcal P}_\text{t}$, $\boldsymbol{\Omega}_\text{s}$, $\boldsymbol{\Omega}_\text{t}$ & Prompt generators and parameters. \\

$\mathbf p_\text{s}^{\mathcal T}$, $\mathbf p_\text{t}^{\mathcal T}$ & Semantic and temporal prompts. \\

$\mathbf Z_{\mathbf p}^{\mathcal T}$, $\widetilde{\mathbf Z}_{\mathbf p}^{\mathcal T}(t)$ & Prompted representations. \\

$\operatorname{LN}(\cdot)$, $\operatorname{Linear}(\cdot)$ & Fusion operators. \\

$\overline{\mathbf h}_{t,\mathbf y}^{\mathcal T}$, $\mathcal Y$ & Class prototype and label space. \\

$\mathcal L_\text{node}$, $\mathcal L_\text{link}$, $\mathcal L_\text{task}$, $\mathcal L_\text{ftn}$ & Fine-tuning losses. \\

$\lambda_r$ & Routing regularization weight. \\

\bottomrule

\end{tabular}
    }
  \label{tab:notation}%
\end{table}%